\documentclass{article}
\usepackage{subcaption}
\usepackage{wrapfig}
\usepackage{amsthm}
\usepackage{amsmath}
\usepackage[ruled,vlined]{algorithm2e}
\usepackage{PRIMEarxiv}
\usepackage{placeins}
\usepackage[utf8]{inputenc} 
\usepackage[T1]{fontenc}    
\usepackage{hyperref}       
\usepackage{url}            
\usepackage{booktabs}       
\usepackage{amsfonts}       
\usepackage{nicefrac}       
\usepackage{microtype}      
\usepackage{lipsum}
\usepackage{fancyhdr}       
\usepackage{graphicx}       
\graphicspath{{media/}}     
\newtheorem{theorem}{Theorem}
\newtheorem{proposition}[theorem]{Proposition}
\newtheorem{lemma}[theorem]{Lemma}

\theoremstyle{remark}
\newtheorem{remark}[theorem]{Remark}
\title{Beyond Unimodal Bases: Pullback Geometry for Multimodal Data
}

\author{
  Honglei Brinkmann, Lucas Ng, Georgios Batzolis, Mark Girolami, Carola-Bibiane Schönlieb \\
  University of Cambridge \\
  Cambridge, UK\\
   \And
  Willem Diepeveen \\
  University of California \\
  Los Angeles, USA\\
}

\begin{document}
\maketitle

\begin{abstract}
Data-driven Riemannian geometry provides nonlinear interpolation and
geometric representations of high-dimensional data. For these operations to
be statistically meaningful, paths between observations should
preferentially traverse high-likelihood regions. Existing scalable pullback
constructions typically use a unimodal Gaussian latent distribution, assuming
that the data reside close to a single manifold. For multimodal data, mapping
separated modes or local structures into one Gaussian region can require
substantial transport deformation and compromise the resulting geometry.

We introduce a pullback geometry for data supported on mixtures of
manifolds. Using a latent Gaussian mixture, we define its Riemannian metric
as the matrix square of the responsibility-weighted expected component
precision. The metric is smooth and positive definite and recovers the
existing Gaussian construction in the single-component limit. For structured
overlapping mixtures, we establish conditions under which the log-density is
concave along geodesics, providing a formal connection between the proposed
geometry and paths through high-likelihood regions, and derive the
corresponding local curvature relations.

We instantiate this geometry in a normalizing flow with adaptive mixture
learning, allowing the number of active components to emerge from the data
and supporting component-wise reconstruction and local effective-dimension
estimation. Experiments on synthetic geometric data, a controlled multi-view
image setting with a known reference trajectory, and MNIST show reduced
transport distortion, competitive path support, close reference-trajectory
recovery, and improved interpolation realism. These results extend scalable pullback
geometry beyond datasets that reside close to a single manifold while
retaining tractable and interpretable local structure.
\end{abstract}


\section{Introduction}
\label{sec:introduction}

A data-driven Riemannian geometry equips a dataset with learned notions of
distance, paths, and local structure, enabling geodesic interpolation,
analysis of intrinsic structure, and low-dimensional representation of curved data  \cite{arvanitidis2018latent}
\cite{arvanitidis2018latent,
diepeveen2024pullback}. These operations are meaningful only when the
geometry reflects the data distribution: geodesics between observations in
high-likelihood regions should preferentially traverse other high-likelihood
regions, while remaining computationally tractable and stable.

Recent work has scaled data-driven Riemannian geometry through generative
modelling
\cite{diepeveen2024pullback,diepeveen2024scorepullback}.
Score-based pullback Riemannian geometry (SBPRG) uses an invertible
normalizing flow to map data to a simpler latent distribution, where
geometric operations can be computed efficiently and transferred back by
pullback. Such a flow defines a diffeomorphism, a smooth invertible map with
a smooth inverse, between data and latent space. To retain meaningful
geometry, the learned diffeomorphism must also be sufficiently regular;
existing constructions encourage it to be approximately locally isometric
on the data. Consequently, the flexibility of the flow is restricted, and
the choice of latent statistical model directly affects the resulting
geometry.

Existing scalable constructions typically use a single Gaussian latent
distribution~\cite{diepeveen2024scorepullback}, effectively assuming that
the data reside close to one manifold. For data with multiple modes or local
structures, mapping both the observed structures and the low-density regions
between them into one Gaussian region can require substantial deformation
and compromise the induced geometry.

Figure~\ref{fig:wavyx-main} illustrates this problem on Wavy-X, a synthetic
dataset formed by two intersecting curved structures. A single-Gaussian base
concentrates these structures into one latent region. A fixed two-component
mixture is more flexible but does not alone ensure that the representation
or its geodesics reflect the data structure. In contrast, the adaptive
mixture recovers two distinct latent directions and produces mapped paths
aligned with the observed branches. This motivates jointly learning the
multimodal statistical model and its induced geometry.

\begin{figure*}[h]
    \centering
    \includegraphics[
        width=0.95\textwidth,
        trim=0.3cm 0.3cm 0.21cm 0.2cm,
        clip
    ]{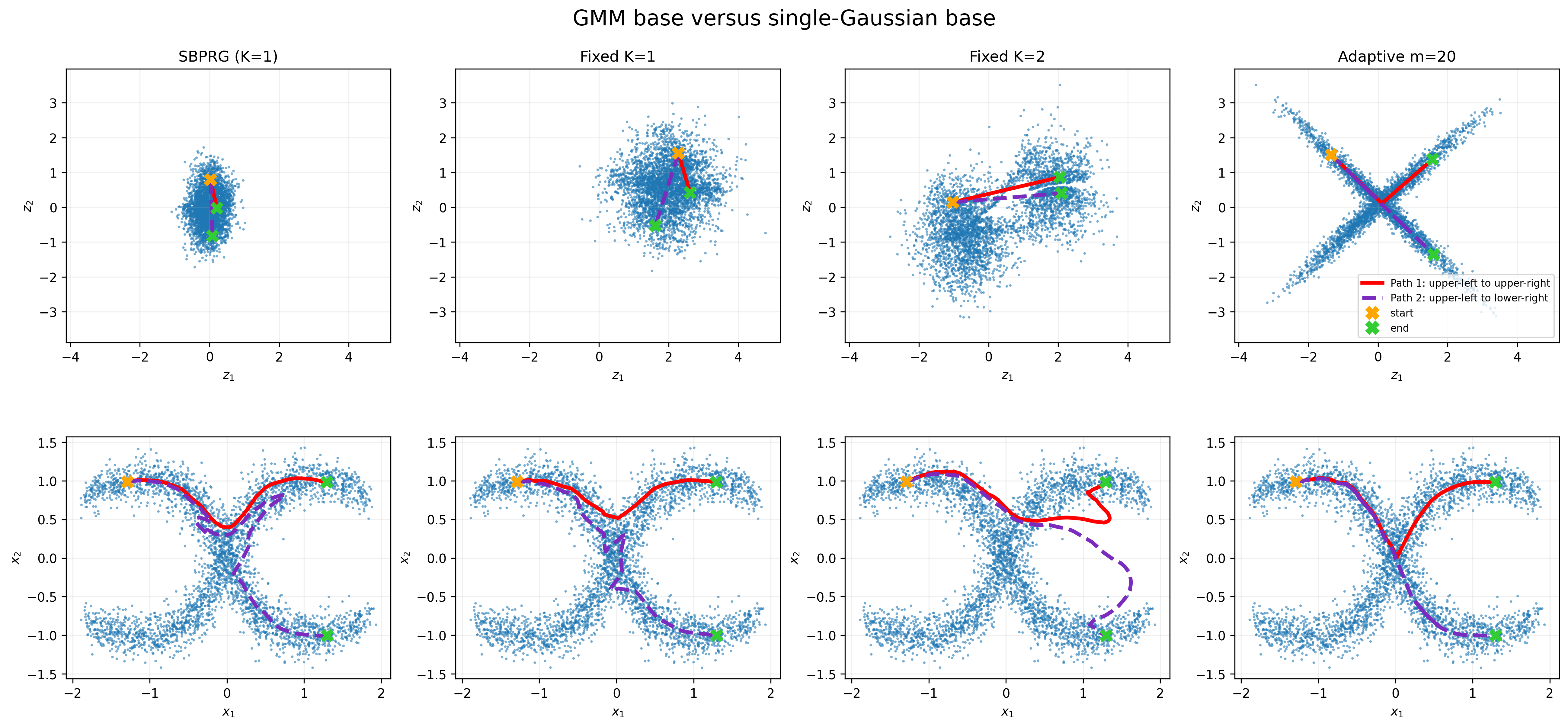}
    \caption{\textbf{Motivating Wavy-X example.}
    Learned latent representations (top) and mapped discrete geodesic
    approximations (bottom) for two shared endpoint pairs. SBPRG and the
    fixed single-Gaussian control concentrate the intersecting structures
    into one latent region. A fixed two-component mixture is more expressive
    but does not consistently recover paths along the observed branches. The
    adaptive mixture retains two active components, organizes the data into
    distinct latent directions, and follows the displayed branches more
    closely.}
    \label{fig:wavyx-main}
\end{figure*}

Extending scalable pullback geometry to mixtures of manifolds therefore
raises three challenges. First, the statistical model must represent
multiple modes without forcing them into a single Gaussian latent region.
Second, the model must induce a computationally tractable Riemannian metric
whose geodesics preferentially pass through high-likelihood regions. Third,
because the number of modes is generally unknown, the learning procedure
should be agnostic to the number of active components.

We address these challenges through the following contributions:
\begin{itemize}
    \item We introduce a normalizing flow with an adaptive Gaussian-mixture
    latent distribution for data supported on mixtures of manifolds.

    \item We define a responsibility-weighted expected-precision metric that
    is smooth, positive definite, computationally tractable, and recovers
    the existing Gaussian construction in the single-component limit. For
    structured shared-mean mixtures, we establish conditions relating its
    geodesics to high-likelihood regions and derive local curvature
    relations.

    \item We jointly learn the flow and mixture while allowing the number of
    active components to emerge during training. The responsibilities also
    support component-wise reconstruction and local effective-dimension
    estimation. Experiments on synthetic geometric data, a controlled
    multi-view image dataset, and MNIST~\cite{lecun1998gradient} evaluate
    transport distortion, path recovery, and interpolation quality.
\end{itemize}

Related work is discussed in Appendix~\ref{app:related_work}.
Section~\ref{sec:background} reviews pullback geometry,
Section~\ref{sec:multimodal_geometry} presents the proposed metric and its
theoretical properties, and Section~\ref{sec:learning} describes adaptive
learning and geometric computation. Implementation details and additional
experiments are provided in the appendix.

\section{Background}
\label{sec:background}

To extend pullback geometry to data supported on mixtures of manifolds, we
first recall how the regularity of a normalizing flow and the choice of its
latent statistical model determine the resulting data geometry. The central
idea of pullback geometry is that, when the learned flow is approximately
locally isometric on the data, geometric computations can be performed in a
tractable latent space and transferred back to data space with controlled
distortion
\cite{diepeveen2024pullback,diepeveen2024scorepullback}. Consequently, the
choice of latent distribution and latent metric directly affects the
geometry induced in data space.
For a Gaussian latent density $p_Z=\mathcal N(\mu,A)$ with covariance
$A\succ0$, the score-based pullback construction uses the constant latent
metric
\begin{equation}
\label{eq:gaussian_score_metric}
g_Z
=
\bigl(D\nabla\log p_Z(z)\bigr)^\top
\bigl(D\nabla\log p_Z(z)\bigr)
=
A^{-2},
\end{equation}
where $D\nabla\log p_Z$ denotes the Hessian of the log-density.

Its constant form makes geometric operations tractable. However, using a
single Gaussian also places the burden of representing multimodal structure
on the flow. This creates a tension: statistical expressiveness favours a
more flexible transport, whereas the geometric guarantees favour a flow
that remains close to a local isometry. Our approach addresses this tension
by increasing the expressiveness of the latent base distribution.

A suitable latent metric should be smooth, positive definite, and
computationally tractable, while producing geodesics that reflect the latent
density. One precise notion of this compatibility is log-concavity of the
density along geodesics
\cite{diepeveen2024scorepullback}. For a geodesic $z(t)$, this requires
\begin{equation}
\label{eq:geodesic_log_concavity}
\frac{\mathrm d^2}{\mathrm dt^2}
\log p_Z\!\left(z(t)\right)
\leq 0.
\end{equation}
This condition prevents the log-density from dipping below the linear
interpolation of its endpoint values along the geodesic and therefore
provides a precise way to study whether paths remain in regions of high
likelihood. It cannot hold globally for every multimodal density,
particularly when modes are strongly separated. We therefore analyze it
for the structured class of overlapping mixtures introduced in
Section~\ref{sec:multimodal_geometry}.

\section{Multimodal Pullback Geometry}
\label{sec:multimodal_geometry}

We seek a latent model for data supported on mixtures of manifolds and a
tractable Riemannian metric whose geodesics preferentially pass through
high-likelihood regions. We therefore use a Gaussian-mixture latent
distribution. However, the score-based construction does not extend straightforwardly from
a single Gaussian to a mixture. Let $z\in\mathbb R^d$, let $K$ be the number
of components, and let $w_k>0$, $\sum_k w_k=1$, with means $\mu_k$ and
covariances $A_k\succ0$. For
$
p_Z(z)=\sum_{k=1}^K w_k\mathcal N(z;\mu_k,A_k),
$
define the responsibilities and expected precision by
\begin{equation}
\label{eq:gmm-density}
\gamma_k(z)
=
\frac{w_k\mathcal N(z;\mu_k,A_k)}{p_Z(z)},
\qquad
P(z)
=
\sum_{k=1}^K\gamma_k(z)A_k^{-1}.
\end{equation}
The corresponding score-based metric is
\begin{equation}
g_{\mathrm{score}}(z)
=
\bigl(D\nabla\log p_Z(z)\bigr)^\top
\bigl(D\nabla\log p_Z(z)\bigr)
=
\left(
-P(z)
+
\operatorname{Cov}_{\gamma(z)}
\left[
\nabla\log\mathcal N(z;\mu_k,A_k)
\right]
\right)^2.
\end{equation}
Here, $\operatorname{Cov}_{\gamma(z)}$ denotes covariance over the component index under the responsibilities $\gamma(z)$. Unlike in the single-Gaussian case, this matrix can be singular, producing a degenerate metric. Where responsibilities are shared across components, the covariance term can partially cancel the expected precision. We therefore omit this term and define
\begin{equation}
\label{eq:expected_precision_metric}
g_Z(z)=P(z)^2,
\end{equation}
which is positive-definite.
When $K=1$, or when components are separated far apart but have identical covariance, it reduces
to $g_Z=A_1^{-2}$, recovering the metric of
\cite{diepeveen2024scorepullback}.
The latter case is important because, for well-separated modes, a path connecting high-density regions must cross a low-density region, so geodesic log-concavity cannot hold everywhere.%
\footnote{Viewed via the desiderata, there is no compelling reason a metric should,
for example, measure distances between isolated modes to route geodesics between them as a shortest-path algorithm would.}
Among the possible choices, our proposal recovers prior work and satisfies the desiderata locally within each component.

\subsection{Star-shaped densities}

The metric in Equation~(\ref{eq:expected_precision_metric}) is defined for
general Gaussian mixtures. To analyze its relationship with the latent
density explicitly, we consider a structured class in which all component
means coincide at the origin and the component precisions share an
eigenbasis. This restriction is used only for the following theoretical
analysis; the learned experimental mixtures need not satisfy it. The
resulting distributions have star-convex superlevel sets, related to the
archetypal setting of \cite{diepeveen2026stararchetype}.

\begin{theorem}[Geodesic log-concavity for Gaussian stars]
\label{thm:star-glc}
For a star with diagonalized covariances, the pointwise log-concavity condition on affine-parametrised geodesic \(z(t)\) under $g_Z=P^2$ is
\begin{equation}
\label{eq:star_glc_condition}
\underbrace{-\dot{z}^{\top}P\dot{z}}_{\leq0}
+
\sum_{1\leq i<j\leq d}
\underbrace{
\operatorname{Cov}_{\gamma(z)}
[q_{k,i},q_{k,j}]
}_{\text{precision covariance}}
\underbrace{
\left(
\sqrt{\frac{\rho_i}{\rho_j}}\dot{z}_i z_j
-
\sqrt{\frac{\rho_j}{\rho_i}}z_i\dot{z}_j
\right)^2
}_{\geq0}
\leq0,
\end{equation}
where $q_{k,i}$ is the $i$-th diagonal element of $A_k^{-1}$ and
$\rho_i=\mathbb E_{\gamma(z)}[q_{k,i}]
=\sum_k\gamma_k(z)q_{k,i}$.
\end{theorem}

One immediate consequence is that geodesic log-concavity holds globally if
$\operatorname{Cov}_{\gamma(z)}[q_{k,i},q_{k,j}]\leq0$ for
$i\neq j$. Geodesics along the coordinate axes are also log-concave with
respect to the density. Restricting the analysis to a coordinate plane
through the origin gives
\begin{equation}
\label{eq:star_glc_2d}
\left(
\frac{z_1^2}{\rho_1(z)}
+
\frac{z_2^2}{\rho_2(z)}
\right)
\operatorname{Cov}_{\gamma(z)}
[q_{k,1},q_{k,2}]
\leq1.
\end{equation}

\begin{wrapfigure}[16]{r}{0.5\textwidth}
    \vspace{-2.5\baselineskip}
    \centering
    \includegraphics[
        width=\linewidth
    ]{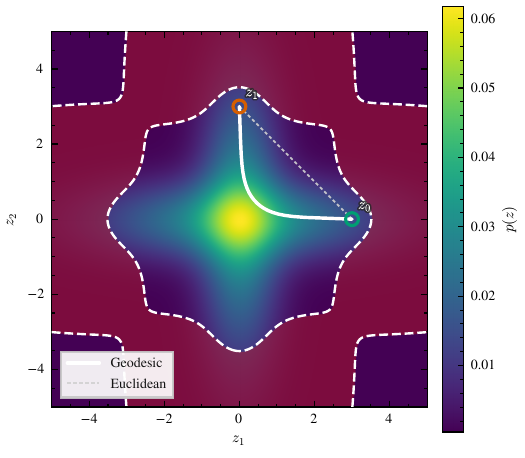}
        \caption{For this $d=2,K=3$ mixture of
    low central covariance 
    \((-0.823)\),
     a large (non-red) region satisfies geodesic log-concavity;
    a geodesic
    follows the arms.}
    \label{fig:safe-zone}
    \vspace{-0.5\baselineskip}
\end{wrapfigure}

For positive precision covariance,
Equation~(\ref{eq:star_glc_2d}) has the form of a generalized ellipse illustrated in Figure~\ref{fig:safe-zone} with
effective radii
\begin{equation}
\begin{aligned}
R_1
&=
\sqrt{
\frac{\rho_1}
{\operatorname{Cov}_{\gamma}[q_{k,1},q_{k,2}]}
},
\\
R_2
&=
\sqrt{
\frac{\rho_2}
{\operatorname{Cov}_{\gamma}[q_{k,1},q_{k,2}]}
}.
\end{aligned}
\end{equation}

\paragraph{Curvature.}
At the origin, the sectional curvature of the coordinate $(i,j)$-plane is
\begin{equation}
\label{eq:sectional_curvature_origin}
K_{ij}(0)
=
\frac{\rho_i+\rho_j}
{\rho_i^2\rho_j^2}
\operatorname{Cov}_{\gamma(0)}
[q_{k,i},q_{k,j}],
\end{equation}
where the $\rho_i$ are evaluated at the origin. Thus, the same precision
covariance controls geodesic log-concavity and sectional curvature. At the
centre of a Gaussian star, all component scores vanish, so
$\nabla P(0)=\nabla g_Z(0)=0$. 

In coordinates normalized such that
$P(0)=I$, the local expansion is
\begin{equation}
\label{eq:local_glc_curvature}
\frac{\mathrm d^2}{\mathrm dt^2}
\log p_Z(z(t))
=
-\|u\|_2^2
+
\frac12
\sum_{i,j}
C_{ij}^0u_i^2z_j^2
+
\frac12R_0(z,u,u,z)
+
\mathcal O(\|z\|^4\|u\|^2)
\end{equation}
where $u$ is the geodesic velocity, \(C_{ij}(z)=\operatorname{Cov}_{\gamma(z)}[q_{k,i},q_{k,j}], C_{ij}^0=C_{ij}(0)\) and \(R_0\) is the Riemann curvature at the origin.
The leading quadratic correction splits into spatial variation of the expected precision and a curvature contribution. The curvature term is generated by cross-coordinate precision covariance.

\paragraph{Informing numerical evaluation.}
We aggregate the sectional curvatures through the Ricci tensor, their
standard contraction, to obtain a compact local summary. Its full expression
and derivation are provided in Appendix~\ref{app:ricci_derivation}. In the
diagonal case, it can be computed in $\mathcal O(Kd)$ time. Writing
$\operatorname{Ric}_\star(0):=\operatorname{Ric}(0)$ for this local structured diagnostic and
$r_i=(\operatorname{Ric}_\star(0))_{ii}$, we report the star--Ricci diagnostic
\begin{equation}
\label{eq:lstar_softplus}
\mathcal L_\star
=
\sum_{i=1}^{d}
\operatorname{softplus}(r_i+\delta_\star),
\end{equation}
where $\delta_\star\geq0$ is a curvature margin; we use $\delta_\star=0$ in the reported
experiments. The quantity is evaluated after training as a local diagnostic.
For learned mixtures with untied means, we report it as a structured
curvature summary rather than as the exact Ricci curvature of the general
learned metric.

Proofs and additional derivations are provided in
Appendix~\ref{app:geometry_proofs}. An additional structured case with
separated component means is discussed in
Appendix~\ref{app:chain_mixtures}.


\section{Stabilized Learning of Multimodal Pullback Geometry}
\label{sec:learning}

We now address three practical requirements: learning without knowing the
mixture complexity, limiting unnecessary transport distortion, and retaining
tractable geodesic computation and local dimension reduction. We alternate
between updating a normalizing flow and an initially overcomplete Gaussian
mixture. Slower mixture updates, damping, and conservative merge--prune
operations stabilize training while allowing the number of active components
to emerge from the data. Full algorithms and regularizer definitions are
provided in Appendix~\ref{app:algorithms}. We write the learned flow as
$\phi_\theta$ and its inverse as $\phi_\theta^{-1}$.

\subsection{Adaptive flow--mixture learning}
\label{sec:adaptive_learning}

The flow is updated frequently on minibatches, whereas mixture parameters and
structural operations are updated less frequently using aggregated latent
statistics. At mixture-update step $e$, we use temperature-smoothed
responsibilities
\begin{equation}
\label{eq:temperature_responsibilities}
\gamma_k^{(T_e)}(z)
=
\operatorname{softmax}_k
\left(
\frac{\log w_k+\log\mathcal N(z;\mu_k,A_k)}{T_e}
\right),
\end{equation}
where $T_e>0$ is annealed to one. Let
$\eta_e=\{w_k,\mu_k,A_k\}_{k=1}^{K_e}$ denote the active mixture parameters
and $\eta_e^*$ their responsibility-weighted moment targets. With damping
coefficient $\zeta\in(0,1]$, we update
\begin{equation}
\label{eq:damped_updates}
\eta_{e+1}
=
(1-\zeta)\eta_e+\zeta\eta_e^*.
\end{equation}
Here, diagonal variances are blended in log space; full covariances are
blended linearly and regularized by a positive diagonal floor.

A symmetric Dirichlet prior regularizes the weights. Its concentration
parameter $\alpha$ favours sparse weights when $\alpha<1$ and balanced
weights when $\alpha>1$. Low-weight components are pruned, and a candidate
merge is accepted only when it improves the penalized mixture objective.
The MAP-style weight update and complete merge--prune procedure are given in
Algorithm~\ref{alg:stable_training}.

\subsection{Training and geometric computation}
\label{sec:training_and_computation}

\paragraph{Training objective.}
The flow is trained using
\begin{equation}
\label{eq:ltotal}
\mathcal L_{\mathrm{train}}
=
\mathcal L_{\mathrm{NLL}}
+
\lambda_{\mathrm{iso}}\mathcal L_{\mathrm{iso}}
+
\lambda_{\mathrm{metric}}\mathcal L_{\mathrm{metric}},
\end{equation}
where
$\mathcal L_{\mathrm{NLL}}
=
-\mathbb E_x
\left[
\log p_Z(\phi_\theta(x))
+
\log|\det D\phi_\theta(x)|
\right]$
and $\lambda_{\mathrm{iso}},\lambda_{\mathrm{metric}}\geq0$. The isometry
term discourages local directional distortion. For MNIST and the BlueWhale rendered-view dataset described in Section~\ref{sec:exp_bluewhale}, the
metric-smoothness term additionally discourages abrupt changes in $P(z)$
between nearby latent points; otherwise $\lambda_{\mathrm{metric}}=0$. Both regularizers are defined in Appendix~\ref{app:regularizers}.

For the three-dimensional synthetic experiment (Appendix~\ref{app:tetrahedral}), MNIST, and BlueWhale, a
small mean-shrinkage update moves selected component means toward their
weight-averaged centre. This update encourages component overlap and is not
an additional term in Equation~(\ref{eq:ltotal}); its definition appears in Appendix~\ref{app:mixture_updates}.

\paragraph{Metric-based geodesics.}
Given endpoints $z_0,z_T\in\mathbb R^d$, let
$z_{0:T}=(z_0,\ldots,z_T)$ be a path with $T$ equal-duration segments. We
optimize its interior vertices using
\begin{equation}
\label{eq:metric_path_energy}
\mathcal E_{g_Z}(z_{0:T})
=
T\sum_{i=0}^{T-1}
\left\|
P\left(\frac{z_i+z_{i+1}}{2}\right)
(z_{i+1}-z_i)
\right\|_2^2.
\end{equation}
Because $g_Z=P^2$, this is the midpoint approximation of the metric energy.
The endpoints are fixed, the interior vertices are initialized linearly, and
the optimized path is mapped to data space through $\phi_\theta^{-1}$.
Wavy-X uses $T=80$, whereas MNIST uses $T=16$. These paths use no density,
smoothness, spacing, arclength, or log-concavity penalty. BlueWhale instead
uses the support-graph path described in
Section~\ref{sec:exp_bluewhale}. Algorithm~\ref{alg:stable_geodesic} gives
the complete continuous solver.

\subsection{Component-wise local reconstruction}
\label{sec:local_reconstruction}

Let $x_n\in\mathbb R^d$ and
$r_{nk}=\gamma_k(\phi_\theta(x_n))$. These responsibilities define weighted
data-space means $\bar x_k$ and covariances $C_k^x$. If
$\nu_{k,1}\geq\cdots\geq\nu_{k,d}$ are the eigenvalues of $C_k^x$, define
\begin{equation}
\label{eq:local_dimension}
d_k
=
\min
\left\{
r:
\frac{\sum_{i=1}^{r}\nu_{k,i}}
{\sum_{i=1}^{d}\nu_{k,i}}
\geq\tau_{\mathrm{var}}
\right\},
\end{equation}
where $\tau_{\mathrm{var}}\in(0,1)$ is the variance-retention threshold.
The leading $d_k$ eigenvectors define a component-wise linear reconstruction
and local effective-dimension estimate. These subspaces are fitted in data
space and are not taken from the latent GMM covariances.


\section{Experiments}
\label{sec:experiments}

We evaluate whether an adaptive Gaussian-mixture latent base reduces flow
distortion and supports data-faithful path computation. The three datasets
evaluate complementary aspects of the method: Wavy-X tests transport
distortion and path support in two dimensions, BlueWhale provides a
metadata-defined reference trajectory, and MNIST evaluates interpolation
realism without assuming a ground-truth path. Additional point-cloud and
component-wise analyses are reported in
Appendices~\ref{app:synthetic} and~\ref{app:image-analysis}.

\subsection{Experimental Setup}
\label{sec:exp_setup}

\paragraph{Datasets.}
\textbf{Wavy-X} is a two-dimensional synthetic dataset formed by two noisy,
sinusoidally perturbed strokes that intersect at the origin and produce four
curved half-arms; see Figure~\ref{fig:wavyx-main}. Generation details are
provided in Appendix~\ref{app:wavyx-generation}.
\textbf{BlueWhale} is a controlled rendered-image dataset whose viewpoint
metadata define a reference trajectory through a star-shaped viewpoint
family; generation details are given in
Appendix~\ref{app:bluewhale-generation}.
\textbf{MNIST} is a high-dimensional multimodal benchmark for which we
assume no ground-truth interpolation path.

\paragraph{Models and evaluation scope.}
We compare the adaptive GMM flow with the single-Gaussian score-based
pullback Riemannian geometry (SBPRG)
baseline~\cite{diepeveen2024scorepullback}. Here, $m$ is the initial number of components; the fixed-$m$ controls
disable adaptation, so $K=m$. The $m=1$ and $m=2$ variants separate the effect of using a mixture base from that of
adaptive component selection. Straight latent interpolation is included as
a path reference where applicable. All models use the same data splits,
RealNVP architecture~\cite{dinh2017density}, and training budget.

The path procedure depends on the question addressed by each dataset. For
Wavy-X, we minimize the expected-precision metric energy in
Equation~(\ref{eq:metric_path_energy}) with $T=80$ segments; MNIST uses the
same objective with $T=16$. BlueWhale instead uses a density-aware,
metric-weighted support graph over encoded training samples, allowing direct
comparison with its metadata-defined reference trajectory. Construction
details are provided in Appendix~\ref{app:bluewhale_path}.

Unless otherwise stated, results are reported for one fixed random seed and
should be interpreted descriptively. Star--Ricci quantities are computed
only as post-training diagnostics. Although the proposed metric reduces to
SBPRG when $K=1$, our fixed-$m=1$ control uses the proposed training
procedure with component adaptation disabled, whereas SBPRG uses the
baseline training procedure. The two models can therefore learn different
transport maps.

\subsection{Wavy-X: Transport Distortion and Path Support}
\label{sec:exp_wavyx}
\FloatBarrier

Wavy-X permits direct inspection of the learned transport and decoded paths.
We compare SBPRG, fixed $m=1$, fixed $m=2$, and the adaptive GMM initialized
with $m=20$ candidate components. Path~1 joins the upper-left and
upper-right half-arms, whereas Path~2 joins the upper-left and lower-right
half-arms.

We report test negative log-likelihood (NLL), mean isometry deviation
$D_{\mathrm{iso}}$, the 95th percentile (Q95) of the flow-Jacobian condition
number $\kappa_J$, and the final active-component count $K_{\mathrm{active}}$.
We also report the star--Ricci diagnostics $L_\star$ and
$\lambda_{\max}(\operatorname{Ric}_\star(0))$, together with path-support
deviation. Definitions are provided in
Appendix~\ref{app:evaluation-metrics}.

\begin{table}[h]
\centering
\caption{\textbf{Wavy-X quantitative comparison.}
Lower is better for NLL, distortion, conditioning, $L_\star$, and support
deviation. Paths are densely sampled discrete geodesic approximations under
$g_Z=P^2$.}
\label{tab:wavyx-main}
\scriptsize
\setlength{\tabcolsep}{3.2pt}
\begin{tabular}{lrrrrccrr}
\toprule
Model
& NLL $\downarrow$
& $D_{\mathrm{iso}}\downarrow$
& Q95 $\kappa_J\downarrow$
& $K_{\mathrm{active}}$
& $L_\star\downarrow$
& $\lambda_{\max}(\mathrm{Ric}_\star(0))$
& Path 1 support $\downarrow$
& Path 2 support $\downarrow$ \\
\midrule
SBPRG ($K=1$)
& 1.5777 & 2.9788 & 40.9951 & 1
& 1.3863 & 0.0000
& 0.3262 & 0.3320 \\
Fixed $m=1$
& 1.6078 & 9.1065 & 28.2596 & 1
& 1.3863 & $\approx 0$
& 0.3396 & \textbf{0.2905} \\
Fixed $m=2$
& 1.6893 & 10.2015 & 28.0523 & 2
& 1.3863 & $\approx 0$
& 1.1756 & 2.5638 \\
Adaptive $m=20$
& \textbf{1.5590} & \textbf{1.8943} & \textbf{4.1490} & 2
& $\mathbf{6.47\!\times\!10^{-17}}$ & \textbf{-37.2766}
& \textbf{0.2794} & 0.3158 \\
\bottomrule
\end{tabular}
\end{table}

The adaptive model obtains the lowest NLL, isometry deviation, and upper-tail
Jacobian conditioning. It also gives the lowest support deviation for
Path~1, while fixed $m=1$ is best for Path~2. Because support deviation
measures proximity to the data cloud, it does not penalize path folding or
distinguish between equally supported routes. The fixed-$m=2$ results show
that introducing a mixture base alone does not guarantee better paths.

The adaptive model reduces 20 candidate components to two. Its largest
star--Ricci eigenvalue is $-37.28$, yielding $L_\star\approx0$. These
quantities are structured post-training diagnostics rather than the exact
Ricci curvature of the learned untied-mean mixture.

Appendix~\ref{app:wavyx-layerwise} shows how the same samples are
progressively transformed from data space to latent space. Compared with the
baselines, the adaptive model better preserves the two intersecting
structures while requiring less transport deformation, consistent with the
distortion measurements. Component-wise analysis recovers the two crossing
arm pairs and assigns each effective dimension one, with $99.41\%$ and
$99.29\%$ explained variance; see Appendix~\ref{app:rae-wavyx}.

\subsection{BlueWhale: Recovery of a Data-Supported Reference Trajectory}
\label{sec:exp_bluewhale}

BlueWhale consists of rendered views with viewpoint metadata defining a
reference trajectory through the common junction of two viewpoint arms.
Unlike Wavy-X, the encoded observations form three thin branches meeting at
a junction that is not locally a smooth manifold. Continuous optimization
under a smooth Riemannian metric can therefore cut across the junction or
between nearby branches instead of following the observed structure.

We instead compute a shortest path on a mutual $k$-nearest-neighbour graph
whose vertices are encoded training observations. Edge costs combine the
proposed metric with a latent-density penalty, producing a density-aware,
metric-weighted support-graph path. Construction details are given in
Appendix~\ref{app:bluewhale_path}.

We compare the adaptive support-graph path with the decoded SBPRG path and
straight latent interpolation. Because the adaptive path consists of
observed samples, this experiment evaluates recovery of a data-supported
reference trajectory rather than providing a controlled comparison of
continuous geodesic solvers.

\begin{figure}[h]
    \centering

    \begin{minipage}[c]{0.06\linewidth}
        \centering\textbf{(a)}
    \end{minipage}%
    \begin{minipage}[c]{0.9\linewidth}
        \centering
        \includegraphics[
            width=\linewidth,
            trim=0cm 0.1cm 0cm 1.15cm,
            clip
        ]{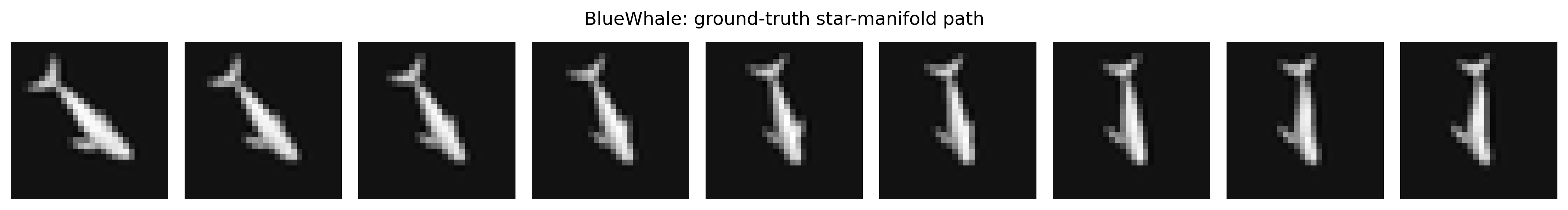}
    \end{minipage}

    \vspace{0.2em}

    \begin{minipage}[c]{0.06\linewidth}
        \centering\textbf{(b)}
    \end{minipage}%
    \begin{minipage}[c]{0.9\linewidth}
        \centering
        \includegraphics[
            width=\linewidth,
            trim=0cm 0.1cm 0cm 1.15cm,
            clip
        ]{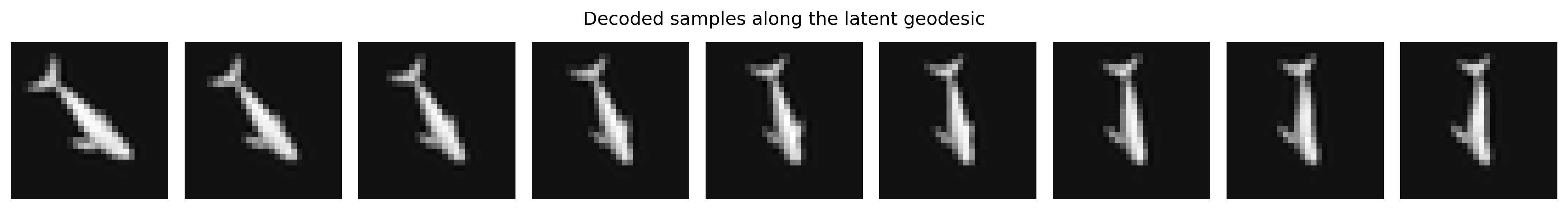}
    \end{minipage}

    \vspace{0.2em}

    \begin{minipage}[c]{0.06\linewidth}
        \centering\textbf{(c)}
    \end{minipage}%
    \begin{minipage}[c]{0.9\linewidth}
        \centering
        \includegraphics[
            width=\linewidth,
            trim=0cm 0.1cm 0cm 1.15cm,
            clip
        ]{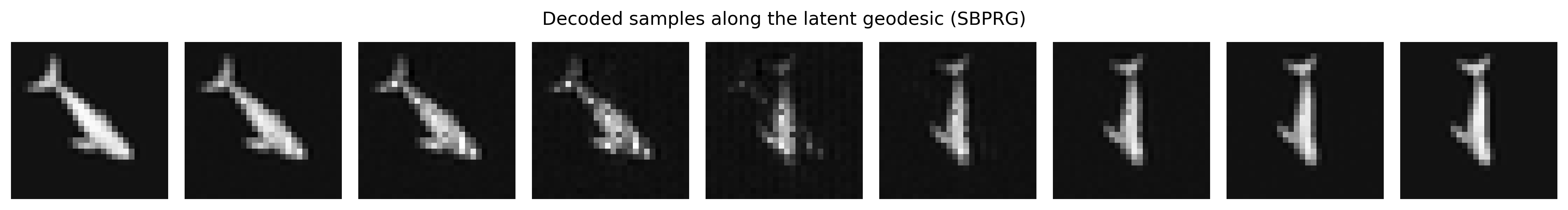}
    \end{minipage}

    \vspace{0.2em}

    \begin{minipage}[c]{0.06\linewidth}
        \centering\textbf{(d)}
    \end{minipage}%
    \begin{minipage}[c]{0.9\linewidth}
        \centering
        \includegraphics[
            width=\linewidth,
            trim=0cm 0.1cm 0cm 1.15cm,
            clip
        ]{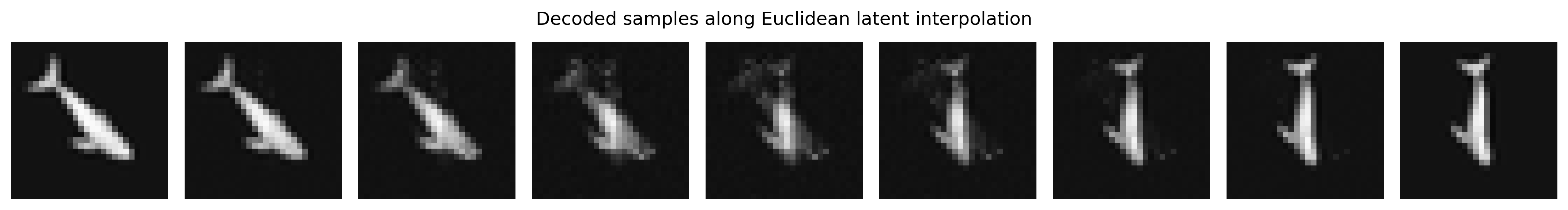}
    \end{minipage}

    \caption{\textbf{BlueWhale viewpoint paths.}
    Image sequences between the same endpoints:
    \textbf{(a)} the metadata-defined reference trajectory;
    \textbf{(b)} the adaptive-GMM support-graph path;
    \textbf{(c)} the decoded SBPRG path; and
    \textbf{(d)} straight latent interpolation.
    Panel~\textbf{(b)} consists of observed samples, whereas
    panels~\textbf{(c)} and~\textbf{(d)} are continuously decoded.}
    \label{fig:bluewhale-paths}
\end{figure}

Figure~\ref{fig:bluewhale-paths} shows that the adaptive representation
admits a support-constrained route aligned with the progressive viewpoint
change in the reference sequence. The continuously decoded SBPRG and
straight-latent paths contain visibly blurred intermediate frames.

At matched sequence positions, we report mean squared error (MSE), mean
absolute error (MAE), mean structural similarity index (SSIM), and
worst-step MSE, defined as the largest frame-wise MSE. Metric and sequence
alignment details are provided in
Appendix~\ref{app:evaluation-metrics}.

\begin{table}[h]
\centering
\caption{\textbf{BlueWhale path comparison.}
Image-space agreement with the metadata-defined reference sequence. The
adaptive path is constrained to observed samples and is therefore not
directly comparable to the continuously decoded alternatives.}
\label{tab:bluewhale-paths}
\small
\setlength{\tabcolsep}{4pt}
\begin{tabular}{lcccc}
\toprule
Path
& MSE $\downarrow$
& MAE $\downarrow$
& Mean SSIM $\uparrow$
& Worst-step MSE $\downarrow$ \\
\midrule
Adaptive-GMM support graph
& 0.000306
& 0.003897
& 0.993647
& 0.001257 \\
SBPRG decoded path
& 0.010547
& 0.032986
& 0.835206
& 0.022814 \\
Straight latent path
& 0.007188
& 0.026982
& 0.859611
& 0.013496 \\
\bottomrule
\end{tabular}
\end{table}

The support-graph path obtains the lowest MSE and MAE, the highest mean SSIM,
and the lowest worst-step MSE relative to the reference sequence. This
demonstrates reference-trajectory recovery on the encoded sample support but
does not isolate the contribution of the continuous expected-precision path
objective. Component-wise reconstruction and effective-dimension results are
reported in Appendix~\ref{app:rae-bluewhale}.

\subsection{MNIST: Interpolation Realism and Perceptual Variation}
\label{sec:exp_mnist}

MNIST provides a multimodal image benchmark without a prescribed
ground-truth path. We compare the adaptive GMM, initialized with $m=30$
candidate components, against SBPRG and straight latent interpolation.
Digit labels are not used to train the generative models; they are used only
to select held-out different-digit endpoint pairs and to train the
independent evaluation classifier. Paths use $T=16$ segments.

We measure interior-sample realism using the Fr\'echet distance (Feature FD)
between all interior path samples and held-out MNIST images in the feature
space of the evaluation classifier. We also report perceptual path length
(PPL), the accumulated feature-space distance between consecutive samples,
and relative perceptual-distance variance
(PDV\textsubscript{rel}), which measures variation in perceptual step size.
Definitions are given in Appendix~\ref{app:evaluation-metrics}.

\begin{figure}[h]
    \centering
    \begin{subfigure}[t]{\linewidth}
        \centering
        \includegraphics[
            width=\linewidth,
            trim=0cm 0.1cm 0cm 0.7cm,
            clip
        ]{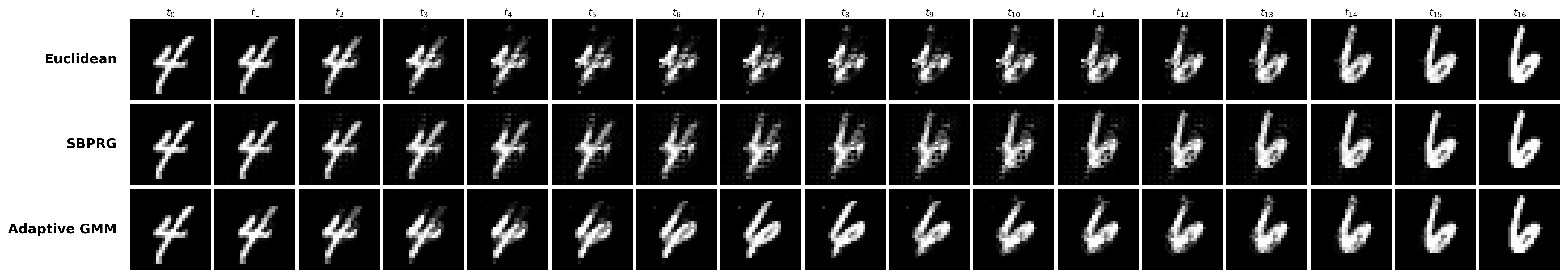}
    \end{subfigure}

    \vspace{0.1em}

    \begin{subfigure}[t]{\linewidth}
        \centering
        \includegraphics[
            width=\linewidth,
            trim=0cm 0.1cm 0cm 0.7cm,
            clip
        ]{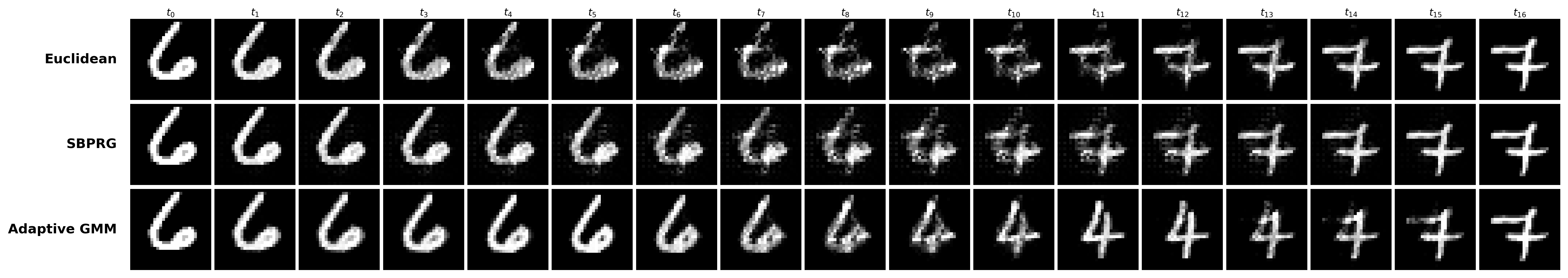}
    \end{subfigure}

    \caption{\textbf{MNIST path interpolation} ($T=16$;
    rows: straight latent / SBPRG / adaptive GMM).
    \textbf{Top:} $4\to6$. The adaptive sequence remains recognizable,
    whereas the alternatives contain visibly noisy intermediate frames.
    \textbf{Bottom:} $6\to7$. The adaptive path passes through a
    digit-4-like region of the learned representation.}
    \label{fig:mnist-geodesics}
\end{figure}

\begin{table}[h]
\centering
\caption{\textbf{MNIST path interpolation.}
Classifier-feature Fr\'echet distance is the primary measure of
interior-sample realism. The adaptive GMM improves this measure at the cost
of greater perceptual path length and step-size variation.}
\label{tab:mnist}
\small
\setlength{\tabcolsep}{6pt}
\begin{tabular}{lccc}
\toprule
Method
& Feature FD $\downarrow$
& PPL $\downarrow$
& PDV\textsubscript{rel} $\downarrow$ \\
\midrule
Straight latent
& 265.218
& 104.667
& 0.166 \\
SBPRG
& 311.615
& \textbf{101.823}
& \textbf{0.120} \\
Adaptive GMM (ours)
& \textbf{144.432}
& 164.663
& 0.426 \\
\bottomrule
\end{tabular}
\end{table}

The adaptive model obtains the lowest classifier-feature Fr\'echet distance,
indicating that its interior samples are closer to the held-out image
distribution in the chosen feature space. Figure~\ref{fig:mnist-geodesics}
is consistent with this result: the adaptive sequences remain recognizable,
and the $6\to7$ example passes through a digit-4-like region.

The higher PPL and PDV\textsubscript{rel} show that the adaptive paths are
longer and less uniform in perceptual feature space. The results therefore
indicate a trade-off: the adaptive geometry favours paths through more
realistic intermediate images, whereas SBPRG produces shorter and more
uniform perceptual transitions.

The component-wise analysis retains 13 active components with effective
dimensions ranging from 17 to 56; reconstruction examples are provided in
Appendix~\ref{app:mnist-rae}.

\paragraph{Summary.}
In this run, the adaptive model improves transport conditioning and gives
the best measured support for one Wavy-X path while remaining competitive
for the other. Its BlueWhale representation supports a graph path closely
aligned with the metadata-defined reference trajectory. On MNIST, its
metric-energy paths improve interior-sample classifier-feature distance at
the cost of greater perceptual length and step-size variation. Together,
the experiments evaluate complementary aspects of the proposed geometry
using dataset-appropriate path procedures.

\section{Conclusion}
\label{sec:conclusion}

Extracting a tractable geometry from a multimodal statistical model requires
more than fitting the data distribution: the induced paths should pass
preferentially through high-likelihood regions while permitting stable
geometric computation. We addressed this challenge using an adaptive Gaussian-mixture latent distribution and a smooth expected-precision metric that recovers the single-Gaussian construction. For shared-mean Gaussian stars, we related geodesic log-concavity to precision covariance and local curvature.

Synthetic and controlled settings provide a necessary intermediate step
toward more complex applications because they allow transport distortion,
path support, and reference-path recovery to be evaluated directly. Across
Wavy-X, BlueWhale, and MNIST, the results show reduced transport distortion,
competitive path support, close reference-trajectory recovery, and improved interpolation realism.
The star--Ricci quantity is used here as a post-training diagnostic; studying its role as a regularizer on more complex data, together with more scalable flow architectures, provides a direction for future work.


\section*{AI Use Statement}

The authors used Claude (Anthropic) for text revision, proofreading, and code documentation. The authors assume responsibility for all content.

\FloatBarrier
\subsection*{Reproducibility statement}

The code used to produce all experimental results in this paper will be made publicly available soon. Implementation and evaluation details are provided in the appendix. Measurements are generated with a fixed seed to ensure consistent evaluation and reproducibility.

\bibliographystyle{unsrt}  
\bibliography{references}  

\clearpage
\appendix

\section{Related Work}
\label{app:related_work}

The central problem is not only to fit a flexible generative distribution,
but to extract from it a geometry that supports statistically meaningful and
computationally tractable operations. To further put this problem into
perspective, it is worthwhile to highlight that the quality of the resulting
geometry directly affects downstream tasks such as nonlinear interpolation,
dimension reduction, and optimization over learned data manifolds. Existing work addresses complementary parts of this problem: generative-geometric methods construct metrics from learned models, pullback methods improve scalability, and multimodal models
increase statistical expressiveness. We review these directions to clarify why their connection remains unresolved.

\paragraph{Data-driven geometry and downstream tasks.}
Deep generative models can induce Riemannian metrics that account for
distortion between latent and data space, providing parametrization-aware
distances and nonlinear interpolation paths
\cite{arvanitidis2018latent,
shao2018riemannian}. Once a reliable geometry has been learned, it can
support nonlinear representation learning through Riemannian autoencoders,
as well as barycentre computation, clustering, and optimization over learned
data manifolds
\cite{diepeveen2024scorepullback,diepeveen2025isoroptimization}.
These applications depend on the quality of the statistical and geometric
models: a metric that assigns low cost to unsupported directions can produce
misleading paths, representations, or optimization trajectories.
Diffusion-based geometry provides an alternative route to estimating
dimension, tangent spaces, curvature, and related quantities directly from
samples
\cite{jones2026manifold,jones2026metricmatching}. Together, these works
show why a reliable geometric foundation is necessary before downstream
manifold operations can be trusted.

\paragraph{Scalable pullback geometry.}
Early generative-geometric methods demonstrated the usefulness of
Riemannian latent spaces but can be expensive to evaluate in high dimensions
\cite{arvanitidis2018latent}. Recent pullback
constructions instead use invertible generative maps to perform geometric
computations in a tractable latent space and transfer them to data space
\cite{diepeveen2024pullback,diepeveen2024scorepullback}. SBPRG is the
closest foundation for our work: an anisotropic Gaussian base induces a
constant precision metric and supports scalable geodesic computation,
dimension estimation, and Riemannian autoencoding. Its statistical model,
however, remains unimodal. Consequently, scalability alone does not address
data containing multiple modes or local manifold-like structures.

\paragraph{Multimodal modelling and density-aware geometry.}
Gaussian-mixture priors have been used in variational autoencoders and
normalizing flows for clustering, semi-supervised learning, and density
estimation
\cite{dilokthanakul2016deep,
izmailov2020flowgmm,wang2023gcflow}. These methods improve statistical
expressiveness but do not by themselves specify the Riemannian metric needed
for geometric data analysis. Multi-chart flows use collections of local maps
to represent nontrivial structures
\cite{kalatzis2021multichart,yu2025multichart}; in contrast, we retain one
invertible flow and derive a spatially varying metric from its Gaussian-
mixture base.

A complementary line of work modifies distances or geometric mappings to
improve their relation to the data distribution. Density-based distances
increase the cost of low-probability regions
\cite{sorrenson2024distances}, while iso-Riemannian methods control
transport distortion and the reliability of geometric operations within a
global pullback construction
\cite{diepeveen2025isoriemannian}. Deformed star distributions provide a
closely related structured setting
\cite{diepeveen2026stararchetype}. These works improve transitions between
regions or control transport regularity; we address an earlier modelling
question: when the data contain multiple modes or local manifold-like
structures, what Riemannian metric should be extracted from the underlying
multimodal statistical model? We answer this by constructing a
positive-definite metric from a Gaussian-mixture base, analyzing its
compatibility with the density for a structured class, and learning the
active components adaptively.

\section{Notation and Geometric Background}
\label{app:geometric_notation}

This section collects the differential- and Riemannian-geometric notation
used throughout the paper. We follow the conventions of
\cite{diepeveen2026stararchetype}, extended here to include the curvature
quantities required in Section~\ref{sec:multimodal_geometry}.

\paragraph{Smooth maps and differentials.}
Let $\mathcal X$ and $\mathcal Z$ be smooth manifolds of dimension $d$,
with tangent spaces $T_x\mathcal X$ and $T_z\mathcal Z$. For a smooth map
$\phi:\mathcal X\to\mathcal Z$, its differential at $x$ is the linear map
\begin{equation}
D_x\phi:
T_x\mathcal X
\longrightarrow
T_{\phi(x)}\mathcal Z.
\end{equation}
Because the spaces considered in this paper are Euclidean, we identify
their tangent spaces with $\mathbb R^d$. In coordinates, $D_x\phi$ is the
Jacobian matrix of $\phi$ at $x$. For a smooth scalar function
$h:\mathbb R^d\to\mathbb R$, we write $\nabla h$ for its Euclidean
gradient and
\begin{equation}
D\nabla h
=
\nabla^2 h
\end{equation}
for its Hessian.

\paragraph{Riemannian metrics and curves.}
A Riemannian metric $g$ assigns to each point $z\in\mathcal Z$ a smoothly
varying inner product
\begin{equation}
g_z:
T_z\mathcal Z\times T_z\mathcal Z
\longrightarrow
\mathbb R.
\end{equation}
In Euclidean coordinates, we identify $g_z$ with a smooth
symmetric positive-definite matrix $g(z)$ and write
\begin{equation}
g_z(u,v)
=
u^\top g(z)v.
\end{equation}
For a continuously differentiable curve
$\gamma:[0,1]\to\mathcal Z$, its Riemannian length and energy are
\begin{align}
L_g(\gamma)
&=
\int_0^1
\sqrt{
\dot\gamma(t)^\top
g(\gamma(t))
\dot\gamma(t)
}
\,dt,
\\
E_g(\gamma)
&=
\frac{1}{2}
\int_0^1
\dot\gamma(t)^\top
g(\gamma(t))
\dot\gamma(t)
\,dt.
\end{align}
A geodesic between fixed endpoints is a locally length-minimizing curve;
under constant-speed parametrization it is equivalently a stationary point
of $E_g$.

\paragraph{Levi--Civita connection and geodesics.}
Every Riemannian metric determines a unique torsion-free,
metric-compatible connection $\nabla$, called the Levi--Civita connection.
In local coordinates, its Christoffel symbols are
\begin{equation}
\Gamma^k_{ij}
=
\frac{1}{2}
g^{k\ell}
\left(
\partial_i g_{j\ell}
+
\partial_j g_{i\ell}
-
\partial_\ell g_{ij}
\right),
\end{equation}
where $(g^{ij})=g^{-1}$ and repeated indices are summed. A smooth curve is
a geodesic when
\begin{equation}
\ddot\gamma^k(t)
+
\Gamma^k_{ij}(\gamma(t))
\dot\gamma^i(t)\dot\gamma^j(t)
=
0.
\end{equation}

\paragraph{Pullback geometry.}
Let $\phi:\mathcal X\to\mathcal Z$ be a diffeomorphism and let $g_Z$ be a
Riemannian metric on $\mathcal Z$. The pullback metric on $\mathcal X$ is
defined by
\begin{equation}
(\phi^*g_Z)_x(u,v)
=
(g_Z)_{\phi(x)}
\left(
D_x\phi[u],
D_x\phi[v]
\right).
\label{eq:appendix_pullback_definition}
\end{equation}
In matrix form,
\begin{equation}
g_X(x)
=
D_x\phi^\top
g_Z(\phi(x))
D_x\phi.
\label{eq:appendix_pullback_matrix}
\end{equation}
Thus, $\phi:(\mathcal X,g_X)\to(\mathcal Z,g_Z)$ is an isometry by
construction. In particular, the length of a curve in data space equals
the length of its image in latent space:
\begin{equation}
L_{g_X}(\gamma)
=
L_{g_Z}(\phi\circ\gamma).
\end{equation}
Consequently, latent geometric computations can be transferred to data
space using $\phi^{-1}$. In this paper, $\phi$ is the learned invertible
normalizing flow and
\begin{equation}
g_Z(z)=P(z)^2
\end{equation}
is the expected-precision metric.

\paragraph{Riemannian Hessian and geodesic log-concavity.}
For a smooth scalar function $h:\mathcal Z\to\mathbb R$, its Riemannian
Hessian is
\begin{equation}
\operatorname{Hess}_g h(X,Y)
=
g\left(
\nabla_X\operatorname{grad}_g h,
Y
\right).
\end{equation}
If $\gamma$ is a geodesic, then
\begin{equation}
\frac{d^2}{dt^2}h(\gamma(t))
=
\operatorname{Hess}_g h_{\gamma(t)}
\left(
\dot\gamma(t),
\dot\gamma(t)
\right).
\end{equation}
We say that a positive density $p$ is log-concave along a geodesic
$\gamma$ when
\begin{equation}
\frac{d^2}{dt^2}
\log p(\gamma(t))
\leq 0
\end{equation}
along the path. This is the notion of geodesic log-concavity used in
Section~\ref{sec:multimodal_geometry}.

\paragraph{Riemann curvature.}
We use the curvature convention
\begin{equation}
R(X,Y)Z
=
\nabla_X\nabla_Y Z
-
\nabla_Y\nabla_X Z
-
\nabla_{[X,Y]}Z.
\label{eq:curvature_convention}
\end{equation}
For linearly independent tangent vectors $u,v\in T_z\mathcal Z$, the
sectional curvature of the plane they span is
\begin{equation}
K_z(u,v)
=
\frac{
g_z(R(u,v)v,u)
}{
g_z(u,u)g_z(v,v)-g_z(u,v)^2
}.
\label{eq:sectional_curvature_definition}
\end{equation}
For the coordinate plane spanned by the $i$-th and $j$-th coordinate
vectors, we write this quantity as $K_{ij}(z)$.

The Ricci tensor is the trace of the curvature operator:
\begin{equation}
\operatorname{Ric}_z(u,v)
=
\operatorname{tr}
\left[
w\longmapsto R(w,u)v
\right].
\label{eq:ricci_definition}
\end{equation}
Equivalently, for any $g_z$-orthonormal basis
$\{e_1,\ldots,e_d\}$,
\begin{equation}
\operatorname{Ric}_z(u,v)
=
\sum_{a=1}^{d}
g_z\left(
R(e_a,u)v,e_a
\right).
\end{equation}
The Ricci tensor therefore contracts the collection of sectional
curvatures into a symmetric bilinear form with only $d^2$ coordinate
entries. Appendix~\ref{app:ricci_derivation} uses this contraction to obtain
a computationally tractable local curvature diagnostic.

\paragraph{Matrix notation.}
For the curvature convention in Equation~(\ref{eq:curvature_convention}),
we use the associated four-tensor
\begin{equation}
R(X,Y,Z,W)=g\!\left(R(X,Y)Z,W\right).
\end{equation}
For a square matrix $M$, we write
\begin{equation}
\operatorname{Sym}(M)
=
\frac{1}{2}(M+M^\top).
\end{equation}
The relations $M\succ0$ and $M\succeq0$ denote positive definiteness and
positive semidefiniteness, respectively. Expectations, variances, and
covariances indexed by $\gamma(z)$ are taken with respect to the discrete
distribution over mixture components with probabilities
$\{\gamma_k(z)\}_{k=1}^K$.

\section{Geometric Derivations}
\label{app:geometry_proofs}

Here we derive the geometric results used in
Section~\ref{sec:multimodal_geometry}.

\subsection{Gaussian-mixture score Hessian}
\label{app:score_hessian}

Let the mixture be
\(
p_Z(z)
=
\sum_{k=1}^{K}
w_k\mathcal N(z;\mu_k,A_k),
\) with \(
\ell_k(z)
=
\log\mathcal N(z;\mu_k,A_k),
\)
and define the component score function
\(
s_k(z)
=
\nabla\ell_k(z)
=
-A_k^{-1}(z-\mu_k).
\)
The responsibilities are, by definition,
\[
\gamma_k(z)
=
\frac{w_k e^{\ell_k(z)}}
{\sum_j w_j e^{\ell_j(z)}},
\]
and hence
\[
\nabla\log p_Z(z)
=
\sum_k\gamma_k(z)s_k(z).
\]

Differentiating responsibilities gives
\[
\nabla\gamma_k(z)
=
\gamma_k(z)
\left(
s_k(z)
-
\sum_j\gamma_j(z)s_j(z)
\right).
\]
Since \(Ds_k(z)=-A_k^{-1}\),
\begin{align}
D\nabla\log p_Z(z)
&=
\sum_k\gamma_k(z)Ds_k(z)
+
\sum_k s_k(z)\nabla\gamma_k(z)^\top
\nonumber\\
&=
-\sum_k\gamma_k(z)A_k^{-1}
+
\operatorname{Cov}_{\gamma(z)}[s_k(z)].
\label{eq:score_hessian_derivation}
\end{align}
Writing
\(
P(z)
=
\sum_k\gamma_k(z)A_k^{-1},
\)
we obtain
\[
D\nabla\log p_Z(z)
=
-P(z)
+
\operatorname{Cov}_{\gamma(z)}
\left[
\nabla\log\mathcal N(z;\mu_k,A_k)
\right].
\]

\subsection{Validity of the expected-precision metric}
\label{app:metric_validity}

\begin{proposition}
Suppose \(w_k>0\) and \(A_k\succ0\) for all \(k\). Then
\(
g_Z(z)
=
P(z)^2,
\) where \(
P(z)
=
\sum_{k=1}^{K}\gamma_k(z)A_k^{-1},
\)
defines a smooth Riemannian metric on \(\mathbb R^d\).
\end{proposition}

\begin{proof}
As each Gaussian is smooth and strictly positive, \(p_Z(z)>0\) and the
responsibilities \(\gamma_k(z)\) are smooth. Thus \(P\) and \(g_Z=P^2\) are
smooth.
For \(v\neq0\),
\(
v^\top P(z)v
=
\sum_k\gamma_k(z)\,
v^\top A_k^{-1}v
>
0,
\)
hence \(P(z)\succ0\). Since \(P(z)\) is symmetric,
\(
v^\top g_Z(z)v
=
v^\top P(z)^2v
=
\|P(z)v\|_2^2
>
0.
\)
Hence \(g_Z(z)\succ0\).
\end{proof}

\begin{remark}[Single component case]
For \(K=1\) we have
\(
P(z)=A_1^{-1},
\) and \(
g_Z(z)=A_1^{-2},
\)
which is the single-Gaussian score-based pullback metric.
\end{remark}

\subsection{Gaussian stars}
\label{app:star_glc_proof}

Assume that all component means coincide at the origin and that the component
precisions share an eigenbasis. In this basis,
\(
A_k^{-1}
=
\operatorname{diag}(q_{k,1},\ldots,q_{k,d}).
\)
Now, define
\(
\rho_i(z)
=
\mathbb E_{\gamma(z)}[q_{k,i}],
\) and \(
C_{ij}(z)
=
\operatorname{Cov}_{\gamma(z)}
[q_{k,i},q_{k,j}]\).
Then
\(
P(z)
=
\operatorname{diag}(\rho_1,\ldots,\rho_d),
\) and \(
g_{ii}(z)
=
\rho_i(z)^2.
\)

\begin{lemma}
For a Gaussian star,
\(
\partial_j\log p_Z(z)
=
-\rho_j(z)z_j,
\) and \(
\partial_j\rho_i(z)
=
-z_j C_{ij}(z).
\)
Consequently,
\(
\partial_{ij}\log p_Z(z)
=
-\rho_i(z)\delta_{ij}
+
C_{ij}(z)z_i z_j.
\)
\end{lemma}

\begin{proof}
The \(j\)-th coordinate of the \(k\)-th component score is
\(
\partial_j\ell_k(z)
=
-q_{k,j}z_j.
\)
Averaging over responsibilities gives
\(
\partial_j\log p_Z(z)
=
-\rho_j(z)z_j.
\)
Moreover,
\(
\partial_j\gamma_k(z)
=
-\gamma_k(z)
\bigl(q_{k,j}-\rho_j(z)\bigr)z_j.
\)
Therefore
\[
\partial_j\rho_i(z)
=
\sum_kq_{k,i}\partial_j\gamma_k(z)
=
-z_j
\operatorname{Cov}_{\gamma(z)}
[q_{k,i},q_{k,j}],
\]
and differentiating
\(\partial_i\log p_Z=-\rho_i z_i\)
gives the last identity.
\end{proof}

\subsubsection{Geodesic log-concavity}

\addtocounter{theorem}{-1}
\begin{theorem}[Geodesic log-concavity for Gaussian stars (restated)]
\label{thm:star-glc-proof}
Let \(z(t)\) be an affinely parameterized geodesic of \(g_Z=P^2\), with
\(u=\dot z\). Then
\begin{equation}
\frac{\mathrm d^2}{\mathrm dt^2}
\log p_Z(z(t))
=
-u^\top P(z)u
+
\sum_{1\leq i<j\leq d}
C_{ij}(z)
\left(
\sqrt{\frac{\rho_i(z)}{\rho_j(z)}}\,u_i z_j
-
\sqrt{\frac{\rho_j(z)}{\rho_i(z)}}\,z_i u_j
\right)^2.
\label{eq:star_glc_condition_proof}
\end{equation}
\end{theorem}

\begin{proof}
For the diagonal metric \(g_{ii}=\rho_i^2\), clearly
\(
g^{ii}
=
\rho_i^{-2}.
\)
The Christoffel symbols are now
\(
\Gamma^\ell_{ab}
=
\delta_{b\ell}
\frac{\partial_a\rho_\ell}{\rho_\ell}
+
\delta_{a\ell}
\frac{\partial_b\rho_\ell}{\rho_\ell}
-
\delta_{ab}
\frac{\rho_a\partial_\ell\rho_a}{\rho_\ell^2}.
\)
Set \(f=\log p_Z\). Along geodesics in affine parametrization,
\(
\frac{\mathrm d^2}{\mathrm dt^2}f(z(t))
=
u^\top\nabla^2 f\,u
-
\partial_\ell f\,
\Gamma^\ell_{ab}u_a u_b.
\)
The density Hessian yields
\[
u^\top\nabla^2 f\,u
=
-\sum_i\rho_i u_i^2
+
\sum_{a,b}C_{ab}z_a z_bu_a u_b.
\]
Using
\(
\partial_\ell f=-\rho_\ell z_\ell
\)
and
\(
\partial_a\rho_\ell=-z_aC_{\ell a}
\),
the connection term yields
\[
-\partial_\ell f\,
\Gamma^\ell_{ab}u_a u_b
=
-2\sum_{a,b}C_{ab}z_a z_bu_a u_b
+
\sum_{a,b}
C_{ab}
\frac{\rho_a}{\rho_b}
z_b^2u_a^2,
\]
and, hence,
\[
\frac{\mathrm d^2}{\mathrm dt^2}f(z(t))
=
-\sum_i\rho_i u_i^2
+
\sum_{a,b}
C_{ab}
\frac{\rho_a}{\rho_b}
z_b^2u_a^2
-
\sum_{a,b}
C_{ab}z_a z_bu_a u_b.
\]
The diagonal covariance terms cancel. Each off-diagonal pair $i<j$
contributes
\[
C_{ij}
\left(
\frac{\rho_i}{\rho_j}u_i^2z_j^2
+
\frac{\rho_j}{\rho_i}u_j^2z_i^2
-
2u_i u_jz_i z_j
\right),
\]
or, equivalently,
\[
C_{ij}
\left(
\sqrt{\frac{\rho_i}{\rho_j}}\,u_i z_j
-
\sqrt{\frac{\rho_j}{\rho_i}}\,z_i u_j
\right)^2.
\]
Substituting gives Equation~(\ref{eq:star_glc_condition_proof}).
\end{proof}

If $C_{ij}(z)\leq0$ for every $i\neq j$, each pairwise correction in
Equation~(\ref{eq:star_glc_condition_proof}) is nonpositive; hence the
log-density is concave along every geodesic.

In two dimensions, let \(C=C_{12}\). For \(C>0\), the correction term is the rank-one quadratic form
\(
C(a^\top u)^2,
\) where \(
a
=
\begin{pmatrix}
\sqrt{\rho_1/\rho_2}\,z_2\\[2mm]
-\sqrt{\rho_2/\rho_1}\,z_1
\end{pmatrix}.
\)
Satisfying geodesic log-concavity for all velocities \(u\) is hence equivalent to
\(
Caa^\top
\preceq
P.
\)
Since \(P\succ0\), this is again equivalent to
\(
C\,a^\top P^{-1}a
\leq1,
\)
which gives
\[
\left(
\frac{z_1^2}{\rho_1(z)}
+
\frac{z_2^2}{\rho_2(z)}
\right)
C_{12}(z)
\leq1.
\]

The coordinate axes provide an additional case in which strict geodesic
log-concavity holds.

\begin{proposition}[Coordinate-axis geodesics]
\label{prop:axis_geodesics}
For a diagonalized Gaussian star, each coordinate axis is invariant under the
geodesic flow of \(g_Z=P^2\). Every nonconstant geodesic contained in the
\(r\)-th coordinate axis satisfies
\[
\frac{\mathrm d^2}{\mathrm dt^2}
\log p_Z(z(t))
=
-\rho_r(z(t))\dot z_r(t)^2
<0.
\]
\end{proposition}
\begin{proof}
Let \(e_r\) be the $r$-th standard basis vector and set \(z=se_r\). For
\(\ell\neq r\),
\[
\Gamma^\ell_{rr}
=
-\frac{\rho_r\partial_\ell\rho_r}{\rho_\ell^2}.
\]
Since
\(
\partial_\ell\rho_r
=
-z_\ell
\operatorname{Cov}_{\gamma(z)}[q_{k,r},q_{k,\ell}],
\)
we have \(\Gamma^\ell_{rr}=0\) on the \(r\)-th axis. A geodesic initialized
on the axis with velocity parallel to \(e_r\) therefore remains on it.
Along this geodesic, write \(\rho=\rho_r(s)\). The restricted metric and
log-density derivative are
\(
g=\rho(s)^2\,ds^2,
\) and \(
\frac{\mathrm d}{\mathrm ds}\log p_Z(s)
=
-\rho(s)s
\) respectively.
The geodesic equation is
\(
\ddot s
+
\frac{\rho'(s)}{\rho(s)}\dot s^2
=
0,
\)
Hence we simplify via exact cancellation to
\begin{align*}
\frac{\mathrm d^2}{\mathrm dt^2}\log p_Z(s(t))
&=
\bigl[-\rho(s)-s\rho'(s)\bigr]\dot s^2
-\rho(s)s\ddot s\\
&=
-\rho(s)\dot s^2,
\end{align*}
and because \(\rho(s)>0\), the log-density is strictly concave on
nonconstant coordinate-axis geodesics.
\end{proof}

\subsubsection{Sectional curvature}
\label{app:sectional_curvature_centre}
At the common centre, define
\(
R_{ijij}(0)=R_0(e_i,e_j,e_j,e_i).
\)
Then
\(
\partial_j\rho_i(0)
=
0.
\)
All first metric derivative hence vanish at \(z=0\). Differentiating
\(
\partial_j\rho_i=-z_jC_{ij}
\)
once again yields
\(
\partial_\ell\partial_j\rho_i(0)
=
-\delta_{j\ell}C_{ij}(0).
\)
As \(g_{ii}=\rho_i^2\),
\(
\partial_j^2 g_{ii}(0)
=
-2\rho_i C_{ij},
\)
with \(\rho_i\) and \(C_{ij}\) taken at the origin.
Under the curvature convention
\(
R_{ijij}(0)
=
-\frac12
\left(
\partial_j^2g_{ii}(0)
+
\partial_i^2g_{jj}(0)
\right),
\)
we now get
\(
R_{ijij}(0)
=
(\rho_i+\rho_j)C_{ij}.
\)
Since
\(
g_{ii}(0)g_{jj}(0)=\rho_i^2\rho_j^2
\),
the sectional curvature of the coordinate \((i,j)\)-plane is
\[
K_{ij}(0)
=
\frac{\rho_i+\rho_j}
{\rho_i^2\rho_j^2}
\operatorname{Cov}_{\gamma(0)}
[q_{k,i},q_{k,j}].
\]
Hence the pairwise precision covariance governing the adjustment to geodesic log-concavity also
determines the coordinate sectional curvature at a centre.

\subsubsection{Local expansion of geodesic log-density}
\label{app:local_glc_expansion}
In coordinates normalized so that
\(
P(0)=I,
\)
abbreviate
\(
C_{ij}^0
=
C_{ij}(0).
\)
Let $R_0$ denote the Riemann four-tensor evaluated at the origin.
As a (centered) Gaussian star is even in \(z\),
\(
C_{ij}(z)
=
C_{ij}^0
+
\mathcal O(\|z\|^2).
\)
Moreover,
\(
\partial_j\rho_i(0)=0,
\) and \(
\partial_\ell\partial_j\rho_i(0)
=
-\delta_{j\ell}C_{ij}^0.
\)
If we now take the Taylor expansion,
\(
\rho_i(z)
=
1
-
\frac12
\sum_j C_{ij}^0 z_j^2
+
\mathcal O(\|z\|^4).
\)
This means that the first term of Equation~(\ref{eq:star_glc_condition_proof})
yields
\[
-u^\top P(z)u
=
-\|u\|_2^2
+
\frac12
\sum_{i,j}
C_{ij}^0u_i^2z_j^2
+
\mathcal O(\|z\|^4\|u\|^2).
\]
Meanwhile, the pairwise term satisfies
\[
\sum_{i<j}
C_{ij}(z)
\left(
\sqrt{\frac{\rho_i}{\rho_j}}\,u_i z_j
-
\sqrt{\frac{\rho_j}{\rho_i}}\,z_i u_j
\right)^2
=
\sum_{i<j}
C_{ij}^0
(u_i z_j-u_j z_i)^2
+
\mathcal O(\|z\|^4\|u\|^2).
\]

Since, at \(P(0)=I\),
\(
R_{ijij}(0)=2C_{ij}^0,
\)
we hence obtain
\(
\frac12R_0(z,u,u,z)
=
\sum_{i<j}
C_{ij}^0
(u_i z_j-u_j z_i)^2.
\)
Combining expansions results in
\[
\frac{\mathrm d^2}{\mathrm dt^2}
\log p_Z(z(t))
=
-\|u\|_2^2
+
\frac12
\sum_{i,j}
C_{ij}^0u_i^2z_j^2
+
\frac12R_0(z,u,u,z)
+
\mathcal O(\|z\|^4\|u\|^2).
\]
The quadratic correction separates into the local variation of the expected
precision and the curvature contribution generated by cross-coordinate
precision covariance.

\subsubsection{Ricci contraction}
\label{app:ricci_derivation}

For a shared-mean Gaussian mixture centred at the origin, we can derive the
Ricci tensor without assuming that the component precisions share an
eigenbasis. For brevity, let
\(
Q_k=A_k^{-1},
P=P(0)=\mathbb E_{\gamma(0)}[Q_k],
\) and \(
L_k=Q_k-P.
\)
Every component score vanishes at \(z=0\), so
\(
D\gamma_k(0)=0.
\)
Differentiating, once again, gives, for \(u,v\in\mathbb R^d\),
\[
D^2\gamma_k(0)[u,v]
=
-\gamma_k(0)\,
u^\top L_kv.
\]
Also, since
\(
\mathbb E_{\gamma(0)}[L_k]=0
\),
\[
D^2P(0)[u,v]
=
-\mathbb E_{\gamma(0)}
\left[
L_k\,
(u^\top L_kv)
\right].
\]
Because \(DP(0)=0\) and \(g=P^2\),
\[
D^2g(0)[u,v]
=
P\,D^2P(0)[u,v]
+
D^2P(0)[u,v]P.
\]
Define
\(
S_k=PL_k+L_kP.
\)
In coordinates,
\(
g_{ij,ab}(0)
=
-\mathbb E_{\gamma(0)}
\left[
S_{k,ij}L_{k,ab}
\right].
\)
At a point with \(\partial_ag_{ij}=0\), the Ricci tensor is
\[
\operatorname{Ric}_{jl}
=
\frac12g^{ik}
\left(
g_{il,jk}
+
g_{jk,il}
-
g_{ik,jl}
-
g_{jl,ik}
\right).
\]
Substituting \(g^{-1}(0)=P^{-2}\) and the preceding expression for
\(g_{ij,ab}\), the contribution of a fixed component is
\[
\frac12
\left[
\operatorname{Tr}(P^{-2}S_k)L_k
+
\operatorname{Tr}(P^{-2}L_k)S_k
-
L_kP^{-2}S_k
-
S_kP^{-2} L_k
\right].
\]

For the last few steps here, let
\(
C_k
=
P^{-1}Q_k-I
=
P^{-1} L_k,
\) and, as \( L_k\) is symmetric, notice that
\(
PC_k
=
C_k^\top P.
\)
Further notice the identities
\(
\operatorname{Tr}(P^{-2}S_k)
=
2\operatorname{Tr}(C_k),
\) and \(
\operatorname{Tr}(P^{-2} L_k)
=
\operatorname{Tr}(P^{-1}C_k).
\)
Applying these to the above contraction yields
\[
\boxed{
\begin{aligned}
\operatorname{Ric}(0)
={}&
\mathbb E_{\gamma(0)}
\left[
\operatorname{Tr}(C_k)PC_k
+
\operatorname{Tr}(P^{-1}C_k)
P\operatorname{Sym}(C_k)P
\right.\\[-1mm]
&\left.\hspace{19mm}
-
PC_k^2
-
\operatorname{Sym}(PC_k^\top C_k)
\right].
\end{aligned}
}
\]
For diagonalised component precisions, we get an analogous form,
where, by letting
\(
c_{k,i}
=
\frac{q_{k,i}}{\rho_i}-1,
a_k
=
\sum_j c_{k,j},
\) and \(
b_k
=
\sum_j\frac{c_{k,j}}{\rho_j},
\)
the Ricci tensor is
\[
\operatorname{Ric}_{ii}(0)
=
\mathbb E_{\gamma(0)}
\left[
\rho_i c_{k,i}a_k
+
\rho_i^2 c_{k,i}b_k
-
2\rho_i c_{k,i}^2
\right].
\]
The vectors \(c_k\), together with \(a_k\) and \(b_k\), can be accumulated
component-wise, giving all diagonal Ricci entries in \(\mathcal O(Kd)\) time.
This can also be seen from the full general expression by noting that the only operations are matrix multiplications and inverses of diagonal matrices, traces, and sums.

\subsection{Chain-shaped densities}
\label{app:chain_mixtures}

Suppose that all mixture components share covariance \(A\), with means
\(
\mu_k=\alpha_k v
\)
along a fixed nonzero direction \(v\). The expected precision is constant,
\(
P(z)=\sum_k\gamma_k(z)A^{-1}=A^{-1},
\)
so \(g_Z=A^{-2}\) and its geodesics are affine lines.
Let
\(
\sigma_\alpha^2(z)
=
\operatorname{Var}_{\gamma(z)}(\alpha_k).
\)
The component scores satisfy
\(
s_k(z)
=
-A^{-1}z+\alpha_kA^{-1}v,
\)
and Equation~(\ref{eq:score_hessian_derivation}) gives
\[
\nabla^2\log p_Z(z)
=
-A^{-1}
+
\sigma_\alpha^2(z)
(A^{-1}v)(A^{-1}v)^\top.
\]
Hence, for a geodesic with velocity \(u\),
\[
\frac{\mathrm d^2}{\mathrm dt^2}
\log p_Z(z(t))
=
-u^\top A^{-1}u
+
\sigma_\alpha^2(z)
(u^\top A^{-1}v)^2.
\]

Geodesic log-concavity at \(z\) therefore holds in every tangent direction
iff
\[
\sigma_\alpha^2(z)
\sup_{u\neq0}
\frac{
u^\top
(A^{-1}v)(A^{-1}v)^\top
u
}{
u^\top A^{-1}u
}
\leq1.
\]
The supremum can be recognized as a generalized Rayleigh quotient.
Since the numerator is of rank one,
\[
\lambda_{\max}
=
(A^{-1}v)^\top A(A^{-1}v)
=
v^\top A^{-1}v
=
\|v\|_{A^{-1}}^2.
\]
Here, $\|v\|_{A^{-1}}^2:=v^\top A^{-1}v$.
Thus
\[
\operatorname{Var}_{\gamma(z)}(\alpha_k)
\leq
\|v\|_{A^{-1}}^{-2}.
\]
The maximizing generalized eigenvector is parallel to \(v\), so the chain
direction is the first direction in which geodesic log-concavity fails as the
responsibility-weighted spread of component means increases.

This example shows that multimodality does not necessarily require a
non-flat metric. Shared-covariance mixtures induce the same constant metric,
yet their geodesic log-concavity can differ because it depends jointly on
the density and the metric. The following proposition generalizes the
one-dimensional chain construction.

\begin{proposition}[Shared-covariance mixtures]
Suppose \(A_k=A\) for every component and let
\(
\Sigma_\mu(z)
=
\operatorname{Cov}_{\gamma(z)}[\mu_k].
\)
Then \(g_Z=A^{-2}\), and geodesic log-concavity holds in every tangent
direction at \(z\) iff
\(
\lambda_{\max}
\left(
A^{-1/2}\Sigma_\mu(z)A^{-1/2}
\right)
\leq1.
\)
\end{proposition}

\begin{proof}
The metric is constant. Moreover,
\(\nabla^2\log p_Z(z)
=
-A^{-1}
+A^{-1}\Sigma_\mu(z)A^{-1}.
\)
Hence
\[
\sup_{u\neq0}
\frac{
u^\top A^{-1}\Sigma_\mu A^{-1}u
}{
u^\top A^{-1}u
}
=
\lambda_{\max}
\left(
A^{-1/2}\Sigma_\mu A^{-1/2}
\right),
\]
and the result follows.
\end{proof}

\section{Training and Path-Optimization Details}
\label{app:algorithms}

This section provides the adaptive flow--mixture training procedure, the
metric-based geodesic solver, and the definitions of the geometric
regularizers and mixture updates used in the experiments.

\subsection{Adaptive Flow--Mixture Training}

\begin{algorithm}[h]
\scriptsize
\caption{Stabilized Training of an Adaptive Multimodal Flow}
\label{alg:stable_training}

\KwIn{Dataset $\{x_n\}_{n=1}^{N}$; flow $\phi_\theta$; number of epochs
$E$; initial mixture size $m$; Dirichlet parameter $\alpha$ and coefficient
$\lambda_{\mathrm{dir}}$; damping coefficient $\zeta$; responsibility
temperature schedule $\{T_e\}_{e=1}^{E}$; mixture-update interval
$K_{\mathrm{mix}}$; structural warmup $E_0$; pruning threshold
$\varepsilon_{\mathrm p}$; mixture-complexity coefficient $\beta_m$;
merge-acceptance tolerance $\delta_{\mathrm{merge}}$; weight floor
$\epsilon_w$; covariance floor $\epsilon_A$;
optional mean-shrinkage coefficient $s$}

\KwOut{Flow parameters $\theta$ and active mixture parameters
$\{w_k,\mu_k,A_k\}_{k=1}^{K}$}

Initialize $\phi_\theta$ and
$\{w_k,\mu_k,A_k\}_{k=1}^{m}$\;

\For{$e=1,\ldots,E$}{
    Update $\theta$ over minibatches by minimizing
    $\mathcal L_{\mathrm{train}}$ from
    Equation~(\ref{eq:ltotal}), treating the current mixture parameters as
    fixed\;

    \If{$e>E_0$ \textbf{and} $e\bmod K_{\mathrm{mix}}=0$}{
        Compute latent codes $z_n=\phi_\theta(x_n)$\;

        Compute the temperature-smoothed responsibilities
        \begin{equation*}
        r_{nk}^{(T_e)}
        =
        \operatorname{softmax}_k
        \left(
        \frac{
        \log w_k+\log\mathcal N(z_n;\mu_k,A_k)
        }{
        T_e
        }
        \right)
        \end{equation*}
        and effective counts
        $N_k=\sum_n r_{nk}^{(T_e)}$\;

        Compute the MAP-style weight target
        \begin{equation*}
        \widetilde w_k
        =
        \max\{N_k+\lambda_{\mathrm{dir}}(\alpha-1),\epsilon_w\},
        \qquad
        w_k^*
        =
        \frac{\widetilde w_k}
        {\sum_j\widetilde w_j},
        \end{equation*}
        together with the responsibility-weighted targets
        \begin{align*}
        \mu_k^*
        &=
        \frac{1}{N_k}
        \sum_{n=1}^{N}
        r_{nk}^{(T_e)}z_n,
        \\
        A_k^*
        &=
        \frac{1}{N_k}
        \sum_{n=1}^{N}
        r_{nk}^{(T_e)}
        (z_n-\mu_k^*)(z_n-\mu_k^*)^\top.
        \end{align*}

        \If{$s>0$}{
            Apply the mean-shrinkage update in
            Remark~\ref{rem:mean_shrinkage} to the selected target means
            $\{\mu_k^*\}$, using the target weights $\{w_k^*\}$\;
        }

        Apply the damped weight and mean updates
        \begin{align*}
        w_k^+&=(1-\zeta)w_k+\zeta w_k^*,\\
        \mu_k^+&=(1-\zeta)\mu_k+\zeta\mu_k^*.
        \end{align*}
        Renormalize $\{w_k^+\}_{k=1}^{K}$. For diagonal covariances, update
        the variances in log space:
        \begin{equation*}
        \log\operatorname{diag}A_k^+
        =
        (1-\zeta)\log\operatorname{diag}A_k
        +
        \zeta\log\operatorname{diag}A_k^*.
        \end{equation*}
        For full covariances, set
        $A_k^+=(1-\zeta)A_k+\zeta A_k^*+\epsilon_A I$\;

        Prune components whose weights fall below
        $\varepsilon_{\mathrm p}$ and renormalize the remaining weights\;

        Let $p_X$ denote the flow-induced data density. For each candidate
        merge, evaluate
        \begin{equation*}
        S
        =
        -\sum_{n=1}^{N}
        \log p_X(x_n)
        +
        \beta_m K
        -
        \lambda_{\mathrm{dir}}(\alpha-1)
        \sum_{k=1}^{K}\log w_k,
        \end{equation*}
        where $K$ is the current number of active components\;

        Accept the best merge only if it decreases $S$ by more than
        $\delta_{\mathrm{merge}}$\;
    }
}

\Return{$\theta$ and
$\{w_k,\mu_k,A_k\}_{k=1}^{K}$}\;
\end{algorithm}

For a proposed merge of components $a$ and $b$, we use the
moment-preserving parameters
\begin{align}
w_{ab}&=w_a+w_b,\\
\mu_{ab}&=\frac{w_a\mu_a+w_b\mu_b}{w_{ab}},\\
A_{ab}&=\frac{1}{w_{ab}}
\sum_{k\in\{a,b\}}w_k
\left[A_k+(\mu_k-\mu_{ab})(\mu_k-\mu_{ab})^\top\right].
\end{align}
The score $S$ is evaluated with the flow fixed, and the best proposed merge
is accepted only when its score improves on the unmerged model by more than
$\delta_{\mathrm{merge}}$.

\subsection{Geometric Regularizers}
\label{app:regularizers}

Let $J_\theta(x)=D\phi_\theta(x)$ be the Jacobian of the flow, with $x$
sampled from the empirical minibatch distribution. The stochastic
local-isometry regularizer is
\begin{equation}
\label{eq:isometry_regularizer}
\mathcal L_{\mathrm{iso}}
=
\mathbb E_{x,v}
\left[
\left(
\left\lVert J_\theta(x)^\top v\right\rVert_2^2-1
\right)^2
\right],
\end{equation}
where $v\in\mathbb R^d$ is sampled from a standard Gaussian distribution and
normalized to unit Euclidean length. Vector--Jacobian products are evaluated
without explicitly forming the full Jacobian.

For MNIST and BlueWhale, we additionally use a weak metric-smoothness
regularizer. Given minibatch latent codes $\{z_i\}_{i=1}^{B}$, define
\begin{equation}
\label{eq:metric_smoothness_appendix}
\mathcal L_{\mathrm{metric}}
=
\frac{
\sum_{i\neq j}
a_{ij}
\lVert P(z_i)-P(z_j)\rVert_F^2
}{
\sum_{i\neq j}a_{ij}
},
\qquad
a_{ij}
=
\exp\left(
-\frac{\lVert z_i-z_j\rVert_2^2}{h}
\right),
\end{equation}
where $h$ is the median pairwise squared distance within the sampled
minibatch. The weighting emphasizes nearby pairs, so the loss discourages
abrupt local changes in the expected-precision field.

\subsection{Mean-Shrinkage Update}
\label{app:mixture_updates}

\begin{remark}[Mean-shrinkage update]
\label{rem:mean_shrinkage}
For selected experiments, we encourage overlap between a subset
$\mathcal I$ of mixture components by moving their target means slightly
toward their target-weighted centre. We first compute
\begin{equation}
\label{eq:shrinkage_center}
\bar{\mu}
=
\frac{
\sum_{k\in\mathcal I}w_k^*\mu_k^*
}{
\sum_{k\in\mathcal I}w_k^*
},
\end{equation}
and then apply
\begin{equation}
\label{eq:mean_shrinkage}
\mu_k^*
\leftarrow
(1-s)\mu_k^*+s\bar{\mu},
\qquad
k\in\mathcal I,
\end{equation}
where $s\in[0,1]$ is the mean-shrinkage coefficient. Depending on the
experiment, $\mathcal I$ contains either all active components or the two
components with the largest mixture weights. This operation modifies the
mixture update directly and is not an additional term in the flow-training
objective.
\end{remark}

\subsection{Metric-Based Geodesic Computation}

\begin{algorithm}[h]
\small
\caption{Metric-Based Latent Geodesic Computation}
\label{alg:stable_geodesic}

\KwIn{Trained flow $\phi_\theta$ and inverse $\phi_\theta^{-1}$; mixture
parameters $\{w_k,\mu_k,A_k\}_{k=1}^{K}$; endpoints $x_A,x_B$; number of
segments $T$; optimizer step size $\eta_{\mathrm{geo}}$; number of iterations
$M$}

\KwOut{Data-space path $\{x_i\}_{i=0}^{T}$}

Compute $z_0=\phi_\theta(x_A)$ and $z_T=\phi_\theta(x_B)$\;

Initialize
\begin{equation*}
z_i
=
\left(1-\frac{i}{T}\right)z_0
+
\frac{i}{T}z_T,
\qquad
i=1,\ldots,T-1;
\end{equation*}

\For{$m=1,\ldots,M$}{
    \For{$i=0,\ldots,T-1$}{
        Set $\bar z_i=(z_i+z_{i+1})/2$\;

        Compute
        \begin{equation*}
        P(\bar z_i)
        =
        \sum_{k=1}^{K}
        \gamma_k(\bar z_i)A_k^{-1};
        \end{equation*}
    }

    Compute
    \begin{equation*}
    \mathcal E_{g_Z}
    =
    T\sum_{i=0}^{T-1}
    \left\|
    P(\bar z_i)(z_{i+1}-z_i)
    \right\|_2^2;
    \end{equation*}

    Update only $z_1,\ldots,z_{T-1}$ by taking a gradient-based optimizer
    step on $\mathcal E_{g_Z}$ with step size $\eta_{\mathrm{geo}}$\;
}

Map the optimized vertices to data space:
$x_i=\phi_\theta^{-1}(z_i)$ for $i=0,\ldots,T$\;

\Return{$\{x_i\}_{i=0}^{T}$}\;
\end{algorithm}

Algorithm~\ref{alg:stable_geodesic} minimizes only the discrete energy of
the proposed metric $g_Z=P^2$. It contains no density modulation or
auxiliary smoothness, spacing, Euclidean-arclength, or geodesic
log-concavity penalty. Relative to the convention for $E_g$ in
Appendix~\ref{app:geometric_notation}, the constant factor $1/2$ is omitted;
this does not change the minimizing path.

\subsection{BlueWhale Support-Graph Path Computation}
\label{app:bluewhale_path}

The encoded BlueWhale observations concentrate around three thin latent
curves. Direct continuous path optimization can therefore move between
nearby branches through regions containing little or no observed support.
For this experiment, we instead compute a data-supported approximation on
a neighbourhood graph constructed from the full-dimensional latent codes.
The two-dimensional PCA representation is used only for visualization.

Let
\begin{equation}
z_n = \phi_\theta(x_n),
\qquad n=1,\ldots,N,
\end{equation}
denote the encoded training observations. We construct a mutual
$k$-nearest-neighbour graph on $\{z_n\}_{n=1}^N$ using Euclidean distance
only to determine local connectivity. In the reported experiment, we use
$k=40$.

For each retained edge $(i,j)$, let
\begin{equation}
\bar z_{ij}
=
\frac{z_i+z_j}{2},
\qquad
\Delta z_{ij}
=
z_j-z_i.
\end{equation}
Its midpoint approximation to the Riemannian length under the proposed
metric $g_Z(z)=P(z)^2$ is
\begin{equation}
\ell_{ij}
=
\sqrt{
\Delta z_{ij}^{\top}
g_Z(\bar z_{ij})
\Delta z_{ij}
}.
\label{eq:bluewhale_metric_edge}
\end{equation}

To discourage graph transitions through low-density regions, we modulate
this length using the latent negative log-likelihood. Let
$\beta_{\mathrm{NLL}}\geq0$ be the density-penalty strength and let
$c_{\mathrm{clip}}>0$ be a clipping threshold. Define
\begin{equation}
q_{\mathrm{NLL}}(z)
=
\operatorname{clip}
\left(
\frac{
-\log p_Z(z)-m_{\mathrm{NLL}}
}{
s_{\mathrm{NLL}}
},
-c_{\mathrm{clip}},c_{\mathrm{clip}}
\right),
\label{eq:bluewhale_standardized_nll}
\end{equation}
where $m_{\mathrm{NLL}}$ is the median negative log-likelihood over the
encoded training set, $s_{\mathrm{NLL}}$ is the corresponding median
absolute deviation (clipped below by a small positive numerical floor), and
$c_{\mathrm{clip}}=6$. The edge cost is
then
\begin{equation}
C_{ij}^{\mathrm{graph}}
=
\ell_{ij}
\exp\left(
\frac{\beta_{\mathrm{NLL}}}{2}
q_{\mathrm{NLL}}(\bar z_{ij})
\right).
\label{eq:bluewhale_graph_cost}
\end{equation}
We use $\beta_{\mathrm{NLL}}=0.5$. Thus, edges passing through lower-density
regions receive a larger cost, while setting $\beta_{\mathrm{NLL}}=0$
recovers the graph
approximation induced only by the proposed Riemannian metric.

We apply Dijkstra's algorithm to obtain the minimum-cost graph path between
the encoded endpoint observations. Because its vertices correspond to
observed BlueWhale images, the resulting sequence remains on the sampled
data support. We therefore refer to it as a
\emph{density-aware, metric-weighted support-graph path}, rather than an
exact continuous geodesic of $g_Z$.

\section{Experimental Protocol}
\label{app:experimental-details}

\subsection{Dataset Construction}

\subsubsection{Wavy-X}
\label{app:wavyx-generation}

Wavy-X is generated from two sinusoidally perturbed strokes intersecting at
the origin. For each sample, we first draw a stroke indicator
$c\sim\operatorname{Bernoulli}(1/2)$ and a longitudinal coordinate
$t\sim\operatorname{Unif}[-4,4]$. Let $u_c$ denote the unit direction of the
selected stroke, $n_c$ its normal direction, and $s_c\in\{-1,1\}$ the
orientation of its sinusoidal perturbation. The two strokes intersect at an
angle of $90^\circ$. Samples are generated as
\[
x=(t+\epsilon_{\parallel})u_c
  +s_c\,0.6\sin(2.5t)n_c
  +\epsilon_{\perp}n_c,
\]

where
\[
\epsilon_{\parallel}\sim\mathcal N(0,0.04^2),
\qquad
\epsilon_{\perp}\sim\mathcal N(0,0.12^2).
\]
We generate $N=8000$ observations using random seed zero.

\subsubsection{Tetrahedral Star}
\label{app:tetrahedral-generation}

The tetrahedral-star dataset consists of four noisy unit-length arms in
$\mathbb R^3$. Let
\begin{equation}
c_1=(1,1,1)^\top,\quad
c_2=(1,-1,-1)^\top,\quad
c_3=(-1,1,-1)^\top,\quad
c_4=(-1,-1,1)^\top,
\end{equation}
and define the unit arm directions
\begin{equation}
u_k=\frac{c_k}{\lVert c_k\rVert_2},
\qquad k=1,\ldots,4.
\end{equation}
For each arm, we independently draw $n_{\mathrm{arm}}=2000$ longitudinal
coordinates $t_{kn}\sim\operatorname{Uniform}[0,1]$. Let
\begin{equation}
\epsilon_{kn}^{\perp}
=
(I_3-u_ku_k^\top)\xi_{kn},
\qquad
\xi_{kn}\sim\mathcal N(0,\sigma_{\perp}^2I_3),
\end{equation}
and
\begin{equation}
\epsilon_{kn}^{\parallel}=a_{kn}u_k,
\qquad
a_{kn}\sim\mathcal N(0,\sigma_{\parallel}^2),
\end{equation}
with all random variables sampled independently. Each observation is
\begin{equation}
x_{kn}
=
t_{kn}u_k
+
\epsilon_{kn}^{\perp}
+
\epsilon_{kn}^{\parallel}.
\end{equation}
We use $\sigma_{\perp}=0.05$, $\sigma_{\parallel}=0.02$, and random seed
zero, producing $N=4n_{\mathrm{arm}}=8000$ observations. Arm identities are
retained only for evaluation and visualization.

\subsubsection{BlueWhale}
\label{app:bluewhale-generation}

We construct BlueWhale by rendering a fixed whale mesh under viewpoints
parameterized by a three-dimensional star-shaped variable. For each sample,
we select one of three unit arm directions $u_k$, draw
$t\sim\operatorname{Uniform}[0,1]$, and set $\theta=tu_k$. The parameter
$\theta$ determines the object rotation under fixed camera, lighting, and
background settings. Images are rendered with supersampling and downsampled
to $28\times28$ grayscale pixels. We generate 5,000 images and retain the
arm identity and $\theta$ for evaluation only. These
metadata define a reference trajectory between different arms: from the
first endpoint to the common origin and then from the origin to the second
endpoint.

\subsubsection{MNIST}
\label{app:mnist-details}

Digit labels are not used to train the MNIST generative models; they are used
only to select held-out different-digit endpoint pairs, train the independent
evaluation classifier, and organize visualizations. The adaptive model is
initialized with $m=30$ candidate components, and interpolation paths use
$T=16$ segments. Additional metric specifications are
reported below.

\subsection{Evaluation Metrics}
\label{app:evaluation-metrics}

For Wavy-X, we report negative log-likelihood (NLL) and mean local isometry
distortion. The latter is
\begin{equation}
D_{\mathrm{iso}}
=
\mathbb E_x\left[
\frac{1}{d}
\left\|
D\phi_\theta(x)^\top D\phi_\theta(x)-I_d
\right\|_F
\right],
\end{equation}
where $d=2$ and $\phi_\theta$ is the learned flow from data to latent space. Flow
conditioning is summarized by the 95th percentile, denoted Q95, of
\begin{equation}
\kappa_J(x)
=
\frac{\sigma_{\max}(D\phi_\theta(x))}
{\sigma_{\min}(D\phi_\theta(x))}.
\end{equation}
We distinguish the active-component count
\begin{equation}
K_{\mathrm{active}}
=
\sum_{k=1}^{K}\mathbf 1\{w_k\geq\tau\},
\qquad \tau=0.01,
\end{equation}
from the entropy-based effective component count
\begin{equation}
K_{\mathrm{eff}}
=
\exp\left(-\sum_{k=1}^{K}\bar w_k\log\bar w_k\right),
\qquad
\bar w_k=\frac{w_k}{\sum_{j=1}^{K}w_j}.
\end{equation}
The former is integer-valued, whereas the latter measures how evenly the
mixture mass is distributed across components.

For a curvature margin $\delta_\star\geq0$, the star--Ricci diagnostic is
\begin{equation}
L_\star
=
\sum_{i=1}^{d}
\operatorname{softplus}(r_i+\delta_\star),
\qquad
r_i=\operatorname{Ric}_{ii}(0),
\end{equation}
and $\lambda_{\max}(\operatorname{Ric}_\star(0))$ denotes the largest
eigenvalue of the corresponding local star--Ricci matrix. These quantities
are evaluated after training as local diagnostics and are not used to
optimize the reported models. We use $\delta_\star=0$ in all reported
experiments. For learned mixtures with unequal means, these are structured
star diagnostics rather than the exact Ricci curvature of the learned
metric.

After subdividing each optimized latent segment into 20 equal parts and
mapping the resulting vertices to data space, let
$\{\widetilde x_m\}_{m=0}^{M}$ denote the densified path and let
$\mathcal X$ be the sampled data cloud. Its support deviation is
\begin{equation}
D_{\mathrm{support}}
=
\frac{1}{M-1}
\sum_{m=1}^{M-1}
\frac{\min_{x\in\mathcal X}\lVert \widetilde x_m-x\rVert_2}
{s_{\mathcal X}},
\end{equation}
where $s_{\mathcal X}$ is the median nearest-neighbour distance within
$\mathcal X$, excluding each point itself from its nearest-neighbour query.
Increasing the subdivision count from 20 to 40 produces essentially
unchanged results.

\paragraph{BlueWhale.}
At matched positions along each candidate and reference sequence, we report
mean squared error (MSE), mean absolute error (MAE), and the structural
similarity index (SSIM). Worst-step MSE is the largest frame-wise MSE along
the sequence. The metadata-defined reference and the procedure used to
construct the support-graph path are described in
Appendices~\ref{app:bluewhale-generation} and~\ref{app:bluewhale_path}.

\paragraph{MNIST.}
Let $f(x)$ denote the representation produced by the independently trained
MNIST classifier. Let $(m_p,\Sigma_p)$ and $(m_r,\Sigma_r)$ be the empirical
means and covariances of the interior-path and reference-image features,
respectively. Their feature Fr\'echet distance is
\begin{equation}
D_{\mathrm{FD}}
=
\lVert m_p-m_r\rVert_2^2
+
\operatorname{tr}\!\left(
\Sigma_p+\Sigma_r
-2(\Sigma_p^{1/2}\Sigma_r\Sigma_p^{1/2})^{1/2}
\right).
\end{equation}
For a path $\{x_t\}_{t=0}^{T}$, define the perceptual step lengths
\begin{equation}
\ell_t=\lVert f(x_{t+1})-f(x_t)\rVert_2,
\qquad t=0,\ldots,T-1.
\end{equation}
Perceptual path length and relative perceptual-distance variance are
\begin{equation}
\operatorname{PPL}=\sum_{t=0}^{T-1}\ell_t,
\qquad
\operatorname{PDV}_{\mathrm{rel}}
=
\frac{T^{-1}\sum_{t=0}^{T-1}(\ell_t-\bar\ell)^2}
{\bar\ell^2},
\qquad
\bar\ell=\frac1T\sum_{t=0}^{T-1}\ell_t.
\end{equation}
Feature FD is our primary measure of interior-sample realism; PPL and
$\operatorname{PDV}_{\mathrm{rel}}$ measure perceptual length and step-size
uniformity.

\subsection{Component-Wise Linear Reconstruction}
\label{app:component-reconstruction}

For observations $x_n\in\mathbb R^d$, define
$r_{nk}=\gamma_k(\phi_\theta(x_n))$ and $N_k=\sum_n r_{nk}$. The weighted
data-space mean and covariance of component $k$ are
\begin{equation}
\bar x_k=\frac{1}{N_k}\sum_n r_{nk}x_n,
\qquad
C_k^{x}
=
\frac{1}{N_k}\sum_n r_{nk}
(x_n-\bar x_k)(x_n-\bar x_k)^\top.
\end{equation}
Let $\nu_{k,1}\geq\cdots\geq\nu_{k,d}\geq0$ be the eigenvalues of
$C_k^x$. For a variance-retention threshold $\tau_{\mathrm{var}}\in(0,1)$,
the component-wise effective dimension is
\begin{equation}
d_k
=
\min\left\{
r:
\frac{\sum_{i=1}^{r}\nu_{k,i}}
{\sum_{i=1}^{d}\nu_{k,i}}
\geq\tau_{\mathrm{var}}
\right\}.
\end{equation}
If $U_k\in\mathbb R^{d\times d_k}$ contains the leading eigenvectors, the
component-$k$ reconstruction of $x_n$ is
\begin{equation}
\widehat x_{nk}
=
\bar x_k+U_kU_k^\top(x_n-\bar x_k).
\end{equation}
Reported oracle reconstructions select, for each observation, the component
with the smallest reconstruction error. This is a post-training diagnostic,
not an inference-time assignment rule. We use
$\tau_{\mathrm{var}}=0.90$ in the reported analyses.

\section{Additional Synthetic Results}
\label{app:synthetic}

This section provides additional results for the synthetic point-cloud
experiments. Wavy-X illustrates component-wise structure and flow deformation
in two dimensions, while the tetrahedral-star experiment extends the
evaluation to three dimensions. The local reconstruction analyses are
performed after training and are not used to optimize the flow. Latent
responsibilities define weighted component memberships, from which we fit
component-wise linear subspaces in data space. We report the number of
leading directions required to explain a specified fraction of weighted
variance as a local effective dimension, rather than as a formal intrinsic
dimension.

\subsection{Wavy-X}
\label{app:rae-wavyx}

\subsubsection{Component-Wise Reconstruction}

\begin{figure}[h]
    \centering
    \includegraphics[
        width=0.6\linewidth,
        trim=0.2cm 0.2cm 0.1cm 0.2cm,
        clip
    ]{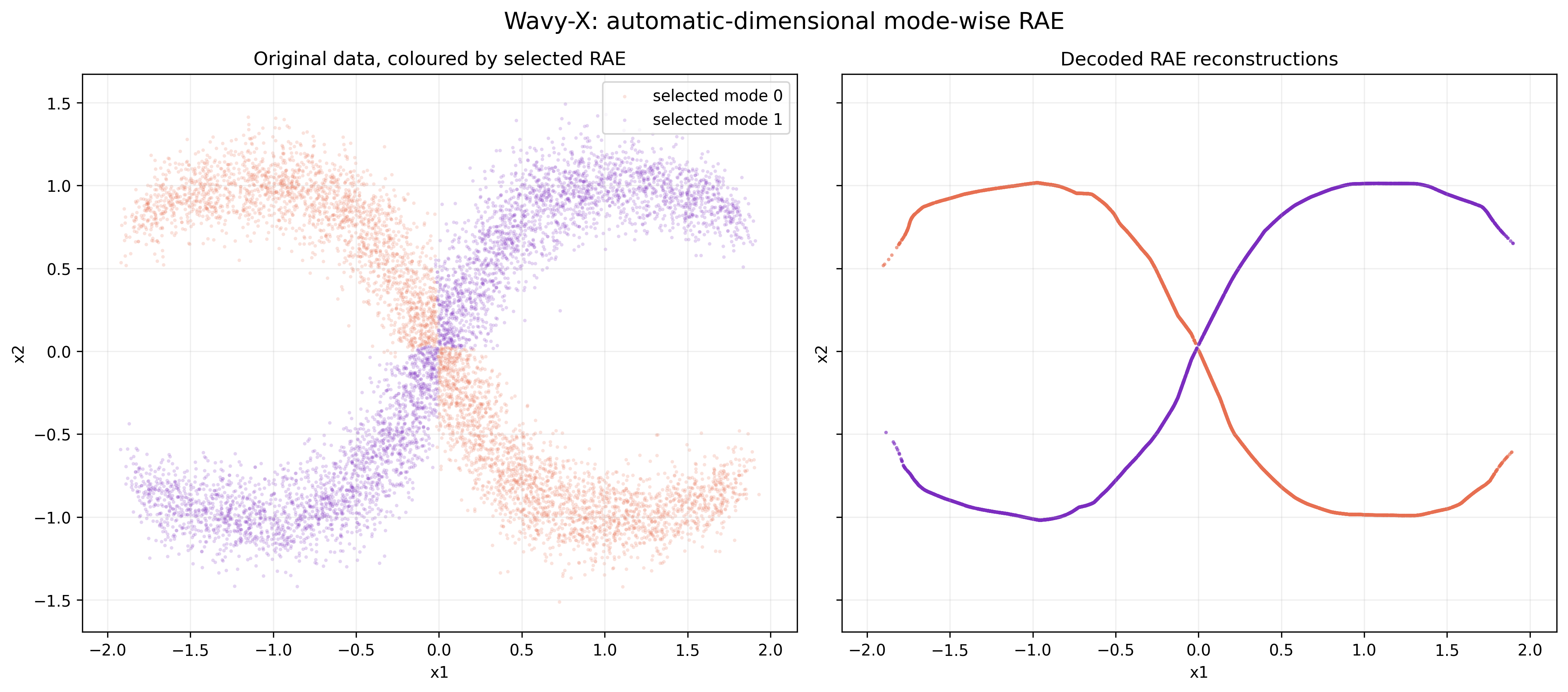}
    \caption{\textbf{Component-wise linear reconstruction on Wavy-X.}
    Original samples and responsibility-weighted linear reconstructions for
    the two retained components. One principal direction explains $99.41\%$
    and $99.29\%$ of their respective weighted data-space variance.}
    \label{fig:wavyx-rae}
\end{figure}

The two retained components have weights $0.5032$ and $0.4968$, and their
dominant directions approximately recover the two crossing arm pairs.
Oracle best-subspace selection, which assigns each sample to the
component-wise reconstruction with the lowest error, gives mean MSE
$1.04\times10^{-2}$ and Q95 MSE $4.17\times10^{-2}$. These results summarize
the dominant component-wise linear structure; they do not establish that
the complete Wavy-X distribution is globally linear or one-dimensional.

\subsubsection{Layer-Wise Transport}
\label{app:wavyx-layerwise}

Figure~\ref{fig:wavyx_layerwise_transport} visualizes how matched Wavy-X
samples evolve through the layers of each normalizing flow. Blue points show
the observations at the input layer, red points show their final latent
representations, and black curves trace individual samples through the
intermediate layers. The figure complements the numerical distortion
measures by showing the different transport deformations produced by the
four latent models.

\begin{figure*}[h]
    \centering
    \includegraphics[
        width=0.7\textwidth
    ]{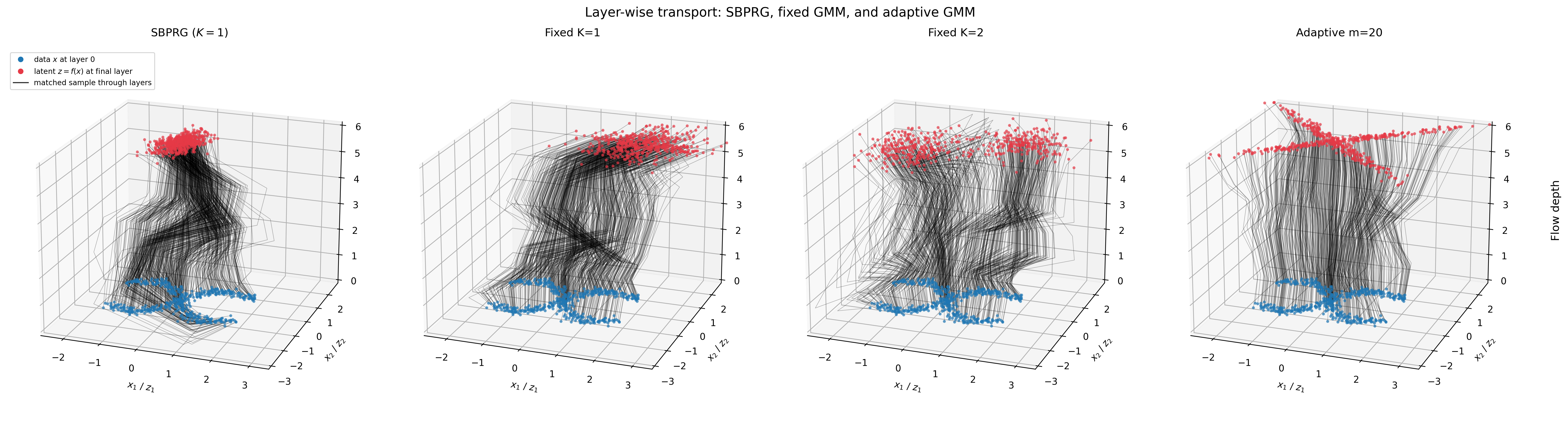}
    \caption{\textbf{Layer-wise transport on Wavy-X.}
    Matched samples are traced from the data space at flow depth zero
    (blue), through the intermediate flow layers (black), to the final latent
    representation (red). From left to right: SBPRG with $K=1$, the fixed $m=1$ single-Gaussian model, the fixed $m=2$ Gaussian-mixture model, and the
    adaptive model initialized with $m=20$ candidate components. The visualization
    shows how each learned flow reorganizes the same Wavy-X observations and
    complements the quantitative transport-distortion comparison.}
    \label{fig:wavyx_layerwise_transport}
\end{figure*}

\FloatBarrier
\subsection{Three-Dimensional Tetrahedral-Star Experiment}
\label{app:tetrahedral}

We further evaluate the proposed method on a three-dimensional
tetrahedral-star dataset consisting of four noisy arms meeting at a common
junction. This extends the Wavy-X setting to a higher-dimensional multimodal
geometry. We compare SBPRG, fixed one- and two-component mixture bases, and
the adaptive mixture-base model. The same endpoint pairs are re-encoded
under each model before computing the corresponding metric paths.

Figure~\ref{fig:tetra-paths} shows the resulting paths in latent and data
space. The adaptive representation preserves the multi-arm organization, and
the mapped paths remain close to the observed support for both selected
endpoint pairs.

\begin{figure}[h]
    \centering
    \includegraphics[
        width=0.75\columnwidth,
        trim=0.3cm 0.5cm 0.3cm 0.1cm,
        clip
    ]{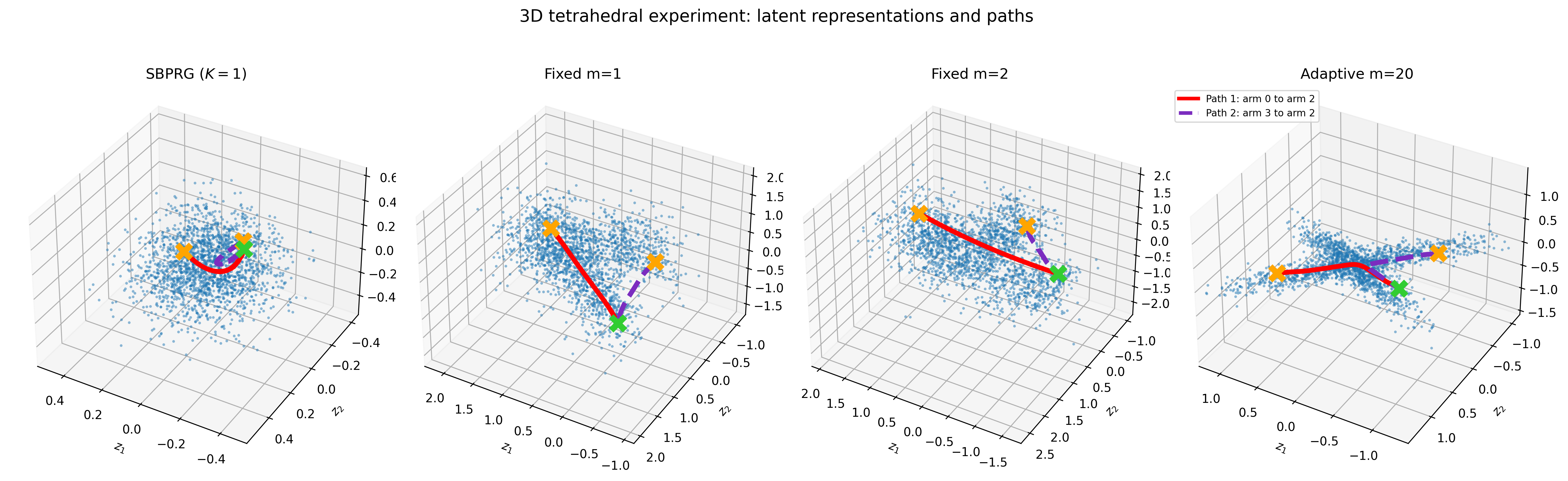}

    \vspace{0.5em}

    \includegraphics[
        width=0.75\columnwidth,
        trim=0.3cm 0.5cm 0.3cm 1.5cm,
        clip
    ]{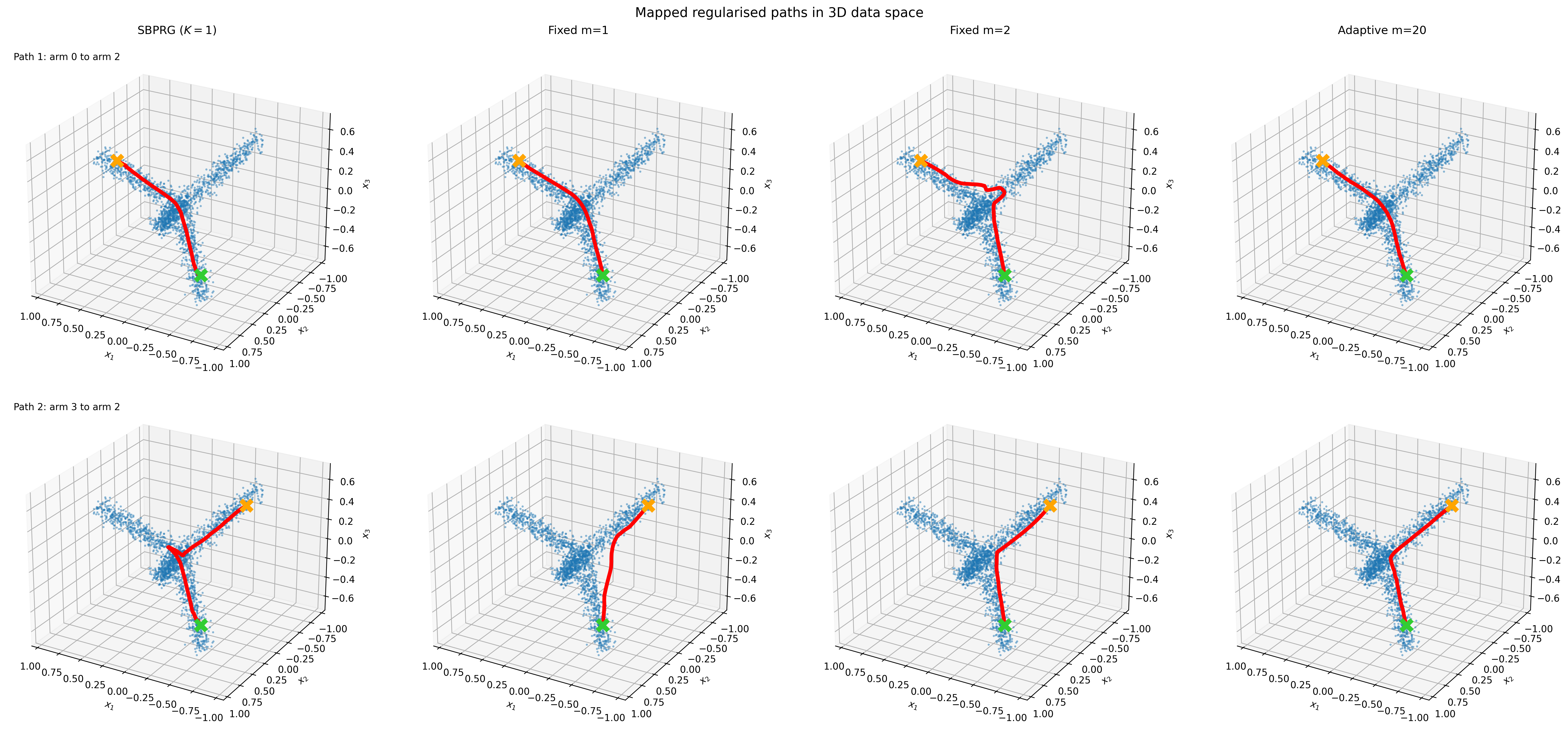}

    \caption{\textbf{Tetrahedral-star metric paths.}
    \textbf{Top:} latent representations and paths.
    \textbf{Bottom:} corresponding paths mapped to data space. Red and
    purple curves denote the two prescribed endpoint pairs.}
    \label{fig:tetra-paths}
\end{figure}

\begin{figure*}[h]
    \centering
    \begin{subfigure}[t]{0.55\textwidth}
        \centering
        \includegraphics[
            width=\linewidth,
            trim=0.2cm 0.6cm 0.2cm 0.7cm,
            clip
        ]{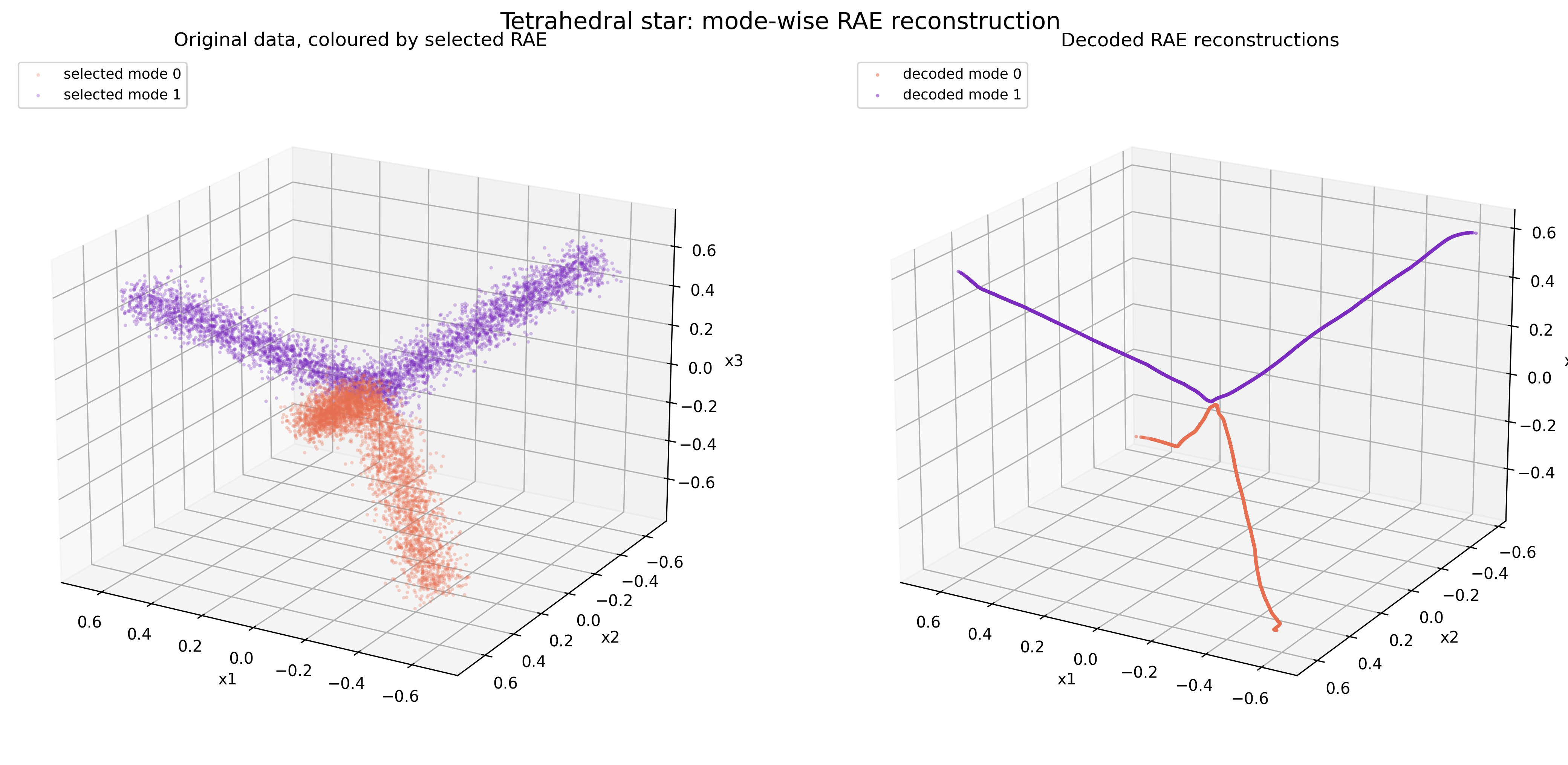}
        \caption{Component-wise local reconstruction.}
        \label{fig:tetra-rae}
    \end{subfigure}

    \vspace{0.5em}

    \begin{subfigure}[t]{0.65\textwidth}
        \centering
        \includegraphics[
            width=\linewidth,
            trim=0.2cm 0.6cm 0.2cm 0.8cm,
            clip
        ]{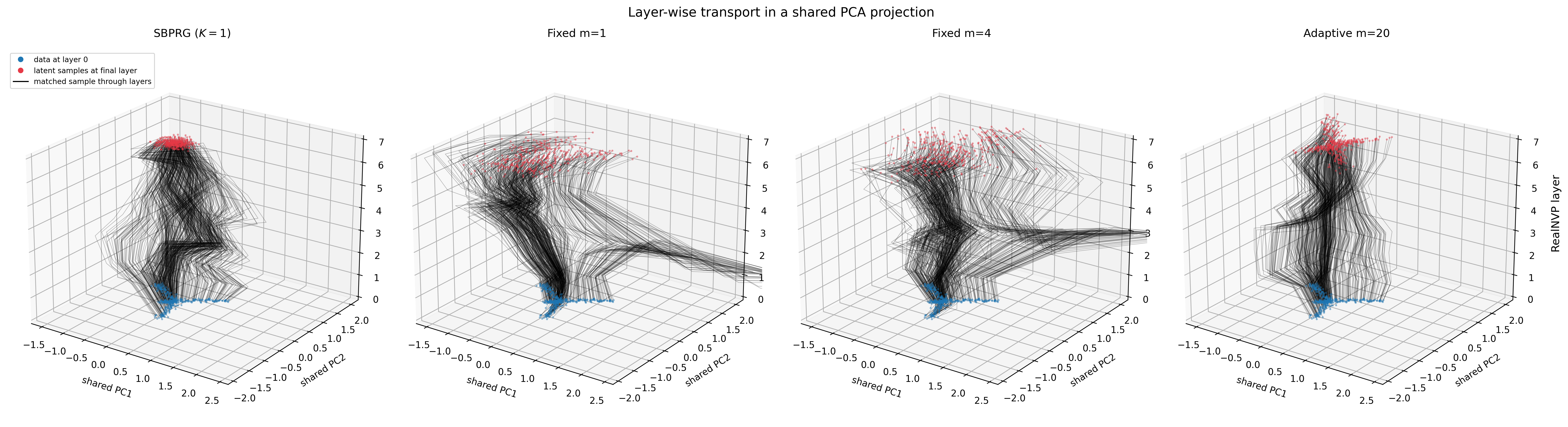}
        \caption{Layer-wise transport in a shared PCA projection.}
        \label{fig:tetra-transport}
    \end{subfigure}

    \caption{\textbf{Additional tetrahedral-star diagnostics.}
    (a) Component-wise reconstructions summarize the dominant local linear
    structure. (b) Layer-wise trajectories visualize how each flow
    transports the data toward its latent base.}
    \label{fig:tetra-diagnostics}
\end{figure*}

Table~\ref{tab:tetrahedral} reports the quantitative results. In this run,
the adaptive model retains approximately two effective components and
obtains the best NLL, isometry distortion, and upper-tail Jacobian condition
number. It obtains the lowest support deviation for Path~2, while SBPRG
obtains the lowest value for Path~1.

\begin{table*}[h]
\centering
\caption{\textbf{Tetrahedral-star quantitative comparison.}
Lower is better for NLL, isometry distortion, upper-tail Jacobian
conditioning, the star--Ricci diagnostic $L_\star$, and support deviation.}
\label{tab:tetrahedral}
\scriptsize
\setlength{\tabcolsep}{3pt}
\resizebox{\textwidth}{!}{%
\begin{tabular}{lcccccccc}
\toprule
Model
& NLL $\downarrow$
& $D_{\mathrm{iso}}\downarrow$
& Q95 $\kappa_J\downarrow$
& $K_{\mathrm{eff}}$
& $L_\star\downarrow$
& $\lambda_{\max}(\mathrm{Ric}_\star(0))$
& Support 1 $\downarrow$
& Support 2 $\downarrow$ \\
\midrule
SBPRG ($K=1$)
& -1.7634 & 4.6871 & 61.5166 & 1.0000 & 2.0794 & 0.0000
& \textbf{0.3970} & 0.4137 \\
Fixed $m=1$
& -1.1533 & 21.3192 & 12.2771 & 1.0000 & 2.0794 & $\approx0$
& 0.4329 & 1.1875 \\
Fixed $m=2$
& -1.1440 & 23.5916 & 11.3507 & 2.0000 & 2.0794 & $\approx0$
& 0.5142 & 0.6308 \\
Adaptive $m=20$
& \textbf{-1.7787} & \textbf{2.8724} & \textbf{5.9996} & 1.9988
& \textbf{$<10^{-6}$} & \textbf{-14.8535}
& 0.4329 & \textbf{0.4067} \\
\bottomrule
\end{tabular}%
}
\end{table*}

\section{Additional Image-Data Analysis}
\label{app:image-analysis}

The following post-training analyses summarize component-wise linear
structure for BlueWhale and MNIST. Component subspaces are fitted in data
space using latent responsibilities and do not alter the learned flow.

\subsection{BlueWhale}
\label{app:rae-bluewhale}

The learned mixture retains three active components with weights $0.2760$,
$0.3182$, and $0.4058$. Under the $90\%$ weighted-variance criterion, their
effective dimensions are $d_k\in\{2,3,2\}$, explaining $96.65\%$, $91.16\%$,
and $90.70\%$ of the respective weighted data-space variance. Although each
generating arm is parameterized by one scalar, its rendered image trajectory
is nonlinear in pixel space and may therefore require more than one linear
direction.

Oracle best-subspace reconstruction gives mean MSE
$5.66\times10^{-4}$ and Q95 MSE $1.46\times10^{-3}$. Here, oracle selection
means that the lowest-error component reconstruction is chosen for each
sample; it is a post-training diagnostic rather than an inference-time
assignment rule.

\FloatBarrier
\subsection{MNIST}
\label{app:mnist-rae}

Figure~\ref{fig:mnist-rae} shows representative reconstructions obtained by
oracle best-subspace selection. The model retains
$K_{\mathrm{active}}=13$ active components, with component-wise effective
dimensions ranging from 17 to 56 (mean 36). Digit~1 has the lowest reported
reconstruction MSE ($4.1\times10^{-3}$), whereas digit~8 has the highest
($1.78\times10^{-2}$). This difference is consistent with the greater
linear complexity of two-loop shapes, although reconstruction errors alone
do not identify topology. Residual blur is expected because each
reconstruction uses a linear approximation in pixel space. Digit labels are
used only to organize the visualization and are not used to fit the
component models.

\begin{figure}[h]
    \centering
    \includegraphics[width=0.9\linewidth]{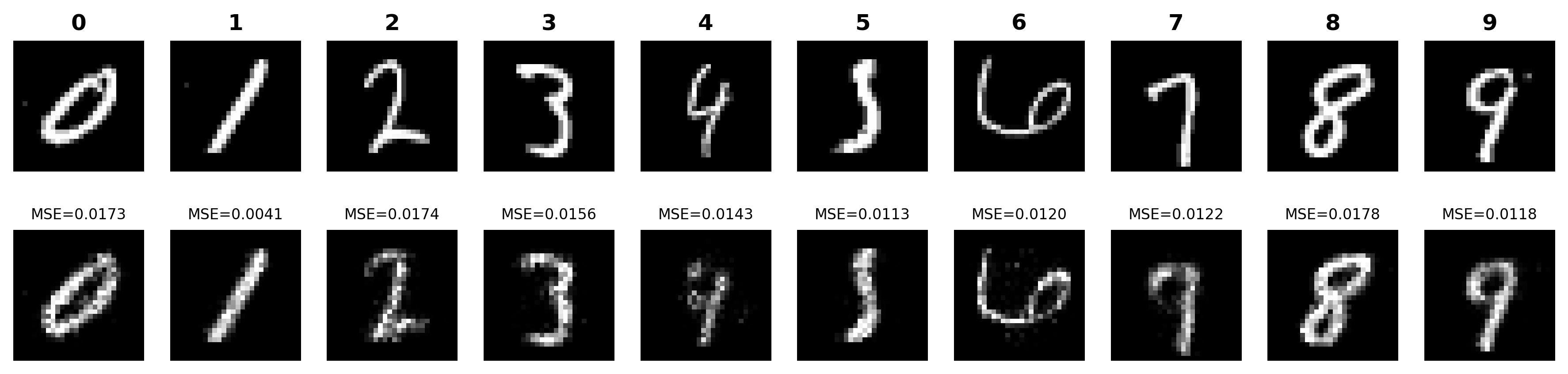}
    \caption{\textbf{Component-wise local reconstruction on MNIST.}
    \emph{Top:} one representative sample per digit class.
    \emph{Bottom:} oracle best-subspace reconstruction with pixel MSE. The
    model retains $K_{\mathrm{active}}=13$ active components from an initial
    pool of 30; digit labels are not used during training or subspace
    fitting.}
    \label{fig:mnist-rae}
\end{figure}

\end{document}